\documentclass[sigconf, nonacm]{acmart}
\AtBeginDocument{%
  \providecommand\BibTeX{{%
    \normalfont B\kern-0.5em{\scshape i\kern-0.25em b}\kern-0.8em\TeX}}}

\usepackage[commandnameprefix=ifneeded]{changes}
\usepackage{graphicx}
\usepackage{subfigure}
\usepackage{hyperref}
\usepackage{changes} 
\definechangesauthor[name={WeiCheng}, color=red]{W}
\usepackage{amsmath}
\usepackage{mathtools}
\usepackage{amsthm}
\usepackage{hyperref}
\usepackage{url}
\usepackage{graphicx,color}
\usepackage{bm}
\usepackage{algorithm}
\usepackage{algorithmic}
\usepackage{subfigure}
\usepackage{multirow}
\usepackage{enumitem,calc}
\usepackage{wrapfig}
\usepackage{lipsum}
\usepackage{todonotes} % provides the \todo{} command, \notice{} etc
\usepackage{thmtools, thm-restate}
\usepackage{pifont}% http://ctan.org/pkg/pifont
\usepackage[normalem]{ulem} 

\usepackage[capitalize,noabbrev]{cleveref}

\newcommand{\cmark}{\ding{51}}

\newtheorem{definition}{Definition}
\newtheorem{assumption}{Assumption}

\newenvironment{remark}[1]{%
  \par\medskip\noindent
  \textbf{Remark #1}
}{%
  \par\medskip
}

\newlist{desclist}{description}{3}
\setlist[desclist,1]{format=\preitem\bfseries,leftmargin=\widthof{\preitem},style=sameline}

\newcommand\preitem{\mdseries\textbullet\space}
\newcommand{\name}{{\textsc{HierTail}}}

\usepackage{expl3}
\ExplSyntaxOn
\newcommand\latinabbrev[1]{
  \peek_meaning:NTF . {% Same as \@ifnextchar
    #1\@}%
  { \peek_catcode:NTF a {% Check whether next char has same catcode as \'a, i.e., is a letter
      #1.\@ }%
    {#1.\@}}}
\ExplSyntaxOff

\def\eg{\latinabbrev{e.g}}

\def\ie{\latinabbrev{i.e}}
\setcopyright{acmcopyright}
\copyrightyear{2024}
\acmYear{2024}
\acmDOI{XXXXXXX.XXXXXXX}

\setcopyright{acmlicensed}\acmConference[KDD '24]{Proceedings of the 30th ACM SIGKDD Conference on Knowledge Discovery and Data Mining}{August 25--29, 2024}{Barcelona, Spain}
\acmISBN{978-1-4503-XXXX-X/18/06}

\begin{document}

\title{Mastering Long-Tail Complexity on Graphs: Characterization, Learning, and Generalization}

\author{Haohui Wang}
\affiliation{%
  \institution{Virginia Tech}
  \country{USA}
}

\author{Baoyu Jing}
\affiliation{%
  \institution{University of Illinois at Urbana Champaign}
  \country{USA}
}

\author{Kaize Ding}
\affiliation{%
  \institution{Northwestern University}
  \country{USA}
}

\author{Yada Zhu}
\affiliation{%
  \institution{IBM Research}
  \country{USA}
}

\author{Wei Cheng}
\affiliation{%
  \institution{NEC Labs}
  \country{USA}
}

\author{Si Zhang}
\affiliation{%
  \institution{Meta}
  \country{USA}
}

\author{Yonghui Fan}
\affiliation{%
  \institution{Arizona State University}
  \country{USA}
}

\author{Liqing Zhang}
\affiliation{%
  \institution{Virginia Tech}
  \country{USA}
}

\author{Dawei Zhou}
\affiliation{%
  \institution{Virginia Tech}
  \country{USA}
}


\begin{abstract}
In the context of long-tail classification on graphs, the vast majority of existing work primarily revolves around the development of model debiasing strategies, intending to mitigate class imbalances and enhance the overall performance.
Despite the notable success, there is very limited literature that provides a theoretical tool for characterizing the behaviors of long-tail classes in graphs and gaining insight into generalization performance in real-world scenarios. 
To bridge this gap, we propose a generalization bound for long-tail classification on graphs by formulating the problem in the fashion of multi-task learning, \ie, each task corresponds to the prediction of one particular class. Our theoretical results show that the generalization performance of long-tail classification is dominated by the overall loss range and the task complexity. 
Building upon the theoretical findings, we propose a novel generic framework \name\ for long-tail classification on graphs. In particular, we start with a hierarchical task grouping module that allows us to assign related tasks into hypertasks and thus control the complexity of the task space; then, we further design a balanced contrastive learning module to adaptively balance the gradients of both head and tail classes to control the loss range across all tasks in a unified fashion. 
Extensive experiments demonstrate the effectiveness of \name\ in characterizing long-tail classes on real graphs, which achieves up to 12.9\% improvement over the leading baseline method in accuracy.
We publish our data and code at~\url{https://anonymous.4open.science/r/HierTail-B961/}.
\end{abstract}






\maketitle

\section{Introduction}
The graph serves as a fundamental data structure for modeling a diverse range of relational data, ranging from financial transaction networks~\cite{wang2019semi, dou2020enhancing} to social science~\cite{fan2019graph}.
In recent years, Graph Neural Networks (GNNs) have achieved outstanding performance on node classification tasks~\cite{Zhou23fast, Yang2023simple, xu2021infogcl} because of their ability to learn expressive representations from graphs. 
Despite the remarkable success, the performance of GNNs is primarily attributed to the availability of high-quality and abundant annotated data~\cite{Yang2020DPGN, garcia2017few, hu2019strategies, Kim2019edge}.
Nevertheless, unlike many graph benchmark datasets developed in the lab environment, it is often the case that many high-stake domains naturally exhibit a long-tail distribution, \ie., a few head classes (the majority classes) 
with rich and well-studied data
and 
 massive tail classes (the minority classes) 
 with scarce
 and under-explored
 data. For example, in financial transaction networks, a few head classes correspond to the normal transaction types (\eg, credit card payment, wire transfer), and the numerous tail classes can represent a variety of fraudulent transaction types (\eg, money laundering, synthetic identity transaction). Despite the rare occurrences of fraudulent transactions, detecting them can prove crucial~\cite{singleton2010fraud, akoglu2015graph}.
Another example is the collaboration network. As shown in Figure~\ref{fig:resultPerClass}, the Cora-Full network~\cite{bojchevski2018deep} encompasses 70 classes categorized by research areas, showcasing a starkly imbalanced data distribution—from as few as 15 papers in the least represented area to as many as 928 papers in the most populated one. The task complexity (massive number of classes, data imbalance) coupled with limited supervision imposes significant computational challenges on GNNs. 

\begin{figure}[t]
  \centering
  \scalebox{1.0}{
  \includegraphics[width=\linewidth]{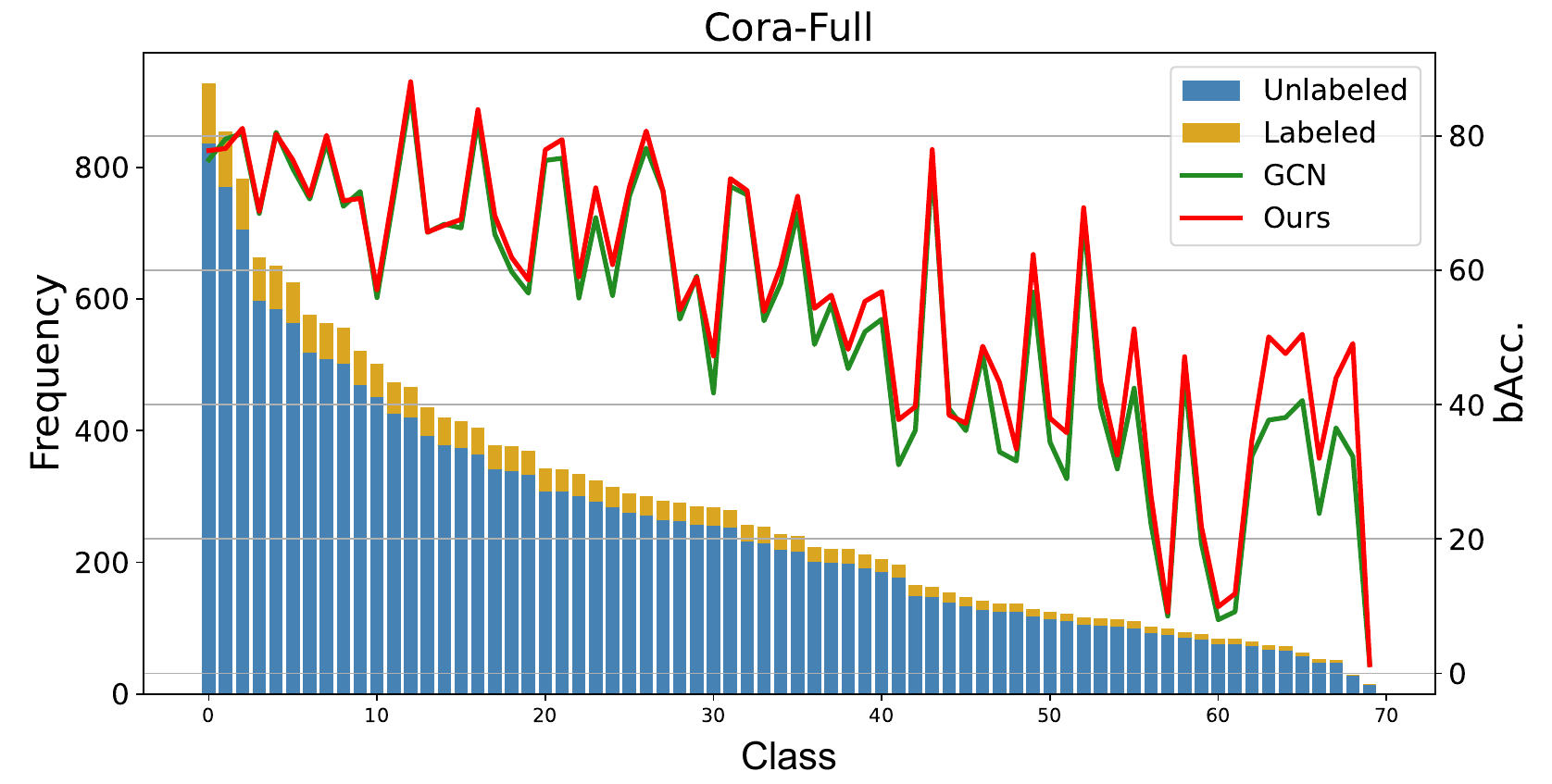}}
  \caption{An illustrative figure of long-tail distribution in the collaboration network (Cora-Full), where the green and red curves show balanced accuracy (bAcc) (\%) of GCN and \name\ for node classification on each class. Blue and yellow bars represent the class frequency of unlabeled and labeled nodes.}
  \label{fig:resultPerClass}
\end{figure}

Important as it could be, there is limited literature that provides a theoretical grounding to characterize the behaviors of long-tail classes on graphs and understand the generalization performance in real environments. 
To bridge the gap, we provide insights and identify three fundamental challenges in the context of long-tail classification on graphs. 
First (\emph{\textbf{C1. Highly skewed data distribution}}), the data exhibits extremely skewed class memberships. Consequently, the head classes contribute more to the learning objective and can be better characterized by GNNs; the tail classes contribute less to the objective and thus suffer from higher systematic errors~\cite{zhang2021deep}.
Second (\emph{\textbf{C2. Label scarcity}}), due to the rarity and diversity of tail classes in nature, it is often more expensive and time-consuming to annotate tail classes  than head classes~\cite{pelleg2004active}. What is worse, training GNNs from scarce labels may result in representation disparity and inevitable errors
~\cite{zhou2019meta, wang2021mixup}, which amplifies the difficulty of debiasing GNN from the highly skewed data distribution. 
Third (\emph{\textbf{C3. Task complexity}}), with the increasing number of classes under the long-tail setting, the difficulty of separating the margin~\cite{hearst1998support} of classes is dramatically increasing. There is a high risk of encountering overlapped regions between classes with low prediction confidence~\cite{zhang2013review, mittal2021decaf}.
To deal with the long-tail classes, the existing literature mainly focuses on augmenting the observed graph~\cite{zhao2021graphsmote, wu2021graphmixup, qu2021imgagn}  or reweighting the class-wise loss functions~\cite{yun2022lte4g, shi2020multi}. 
Despite the existing achievements, a natural research question is that: \emph{can we further improve the overall performance by learning more knowledge from both head classes and tail classes?}

To answer the aforementioned question, we provide the generalization bound of long-tail classification on graphs. The key idea is to formulate the long-tail classification problem in the fashion of multi-task learning~\cite{song2022efficient}, \ie, each task corresponds to the prediction of one specific class. In particular, the generalization bound is in terms of the range of losses across all tasks and the complexity of the task space. Building upon the theoretical findings, we propose \name, a generic learning framework to characterize long-tail classes on graphs. 
Specifically, motivated by controlling the complexity of the task space, we employ a hierarchical structure for task grouping to tackle C2 and C3. It assigns related tasks into hypertasks, allowing the information learned in one class to help train another class, particularly benefiting tail classes.
Furthermore, we implement a balanced contrastive module to address C1 and C2, which effectively balances the gradient contributions across head classes and tail classes. This module reduces the loss of tail tasks while ensuring the performance of head tasks, thus controlling the range of losses across all tasks. 

The main contributions of this paper are summarized below.
\begin{desclist}
\item \textbf{Problem Definition.} We formalize the long-tail classification problem on graphs and develop a novel metric named long-tailedness ratio for characterizing properties of long-tail distributed data. 
\item \textbf{Theory.} We derive a generalization bound for long-tail classification on graphs, which inspires our proposed framework.
\item \textbf{Algorithm.} We propose a novel approach named \name\ that (1) extracts shareable information across classes via hierarchical task grouping and (2) balances the gradient contributions of head classes and tail classes. 
\item \textbf{Evaluation.} We systematically evaluate the performance of \name\ with eleven baseline models on six real-world datasets for long-tail classification on graphs. The results demonstrate the effectiveness of \name\ and verify our theoretical findings. We publish our data and code at an anonymous Github~\footnote{https://anonymous.4open.science/r/HierTail-B961/}.
\end{desclist}

The rest of this paper is organized as follows. We provide the problem definition in Section 2, followed by the proposed framework in Section 3. Section 4 discusses the experimental setup and results, followed by a literature review in Section 5. Finally, we conclude the paper in Section 6.

\section{Preliminary}\label{sec:preliminary}
In this section, we introduce the background and give the formal problem definition. Table~\ref{TB:Notations} in Appendix~\ref{sec:notation} summarizes the main notations used in this paper.
We represent a graph as $\mathcal{G} = (\mathcal{V}, \mathcal{E}, \mathbf{X})$, where $\mathcal{V}$ represents the set of nodes, $\mathcal{E} \subseteq \mathcal{V} \times \mathcal{V}$ represents the set of edges, $\mathbf{X} \in \mathbb{R}^{n \times d}$ represents the node feature matrix, $n$ is the number of nodes, and $d$ is the feature dimension. $\mathbf{A}\in \{0, 1\}^{n \times n}$ is the adjacency matrix, where $\mathbf{A}_{ij} = 1$ if there is an edge $\mathbf{e}_{ij}\in \mathcal{E}$ from $\mathbf{v}_i$ to $\mathbf{v}_j$ in $\mathcal{G}$ and $\mathbf{A}_{ij} = 0$ otherwise. $\mathcal{Y}=\{y_1, \ldots, y_n\}$ is the set of labels, $y_i \in \{1, \ldots, T\}$ is the label of the $i^\text{th}$ node. There are $T$ classes in total, and $T$ can be notably large.

\noindent\textbf{Long-Tail Classification} refers to the classification problem in the presence of a massive number of classes, highly skewed class-membership distribution, and label scarcity. Here we let $\mathcal{D}=\{(\mathbf{x}_i, y_i)\}_{i=1}^{n}$ represent a dataset with long-tail distribution. We define $\mathcal{D}_t$ as the set of instances belonging to class $t$. Without the loss of generality, we have $\mathcal{D} = \{ \mathcal{D}_1, \mathcal{D}_2, \ldots, \mathcal{D}_T\}$, where $|\mathcal{D}_1|\geq |\mathcal{D}_2|\geq\cdots\gg |\mathcal{D}_T|$, $\sum_{t=1}^{T}|\mathcal{D}_t|=n$. Tail classes may encounter label scarcity, having few or even only one instance, while head classes have abundant instances. To measure the skewness of long-tail distribution, \citet{wu2021graphmixup} introduces the Class-Imbalance Ratio as $\frac{\min_t(|\mathcal{D}_t|)}{\max_t(|\mathcal{D}_t|)}$, i.e., the ratio of the size of the smallest minority class to the size of the largest majority class.

\begin{figure}[t]
    \centering
    \includegraphics[width=\linewidth]{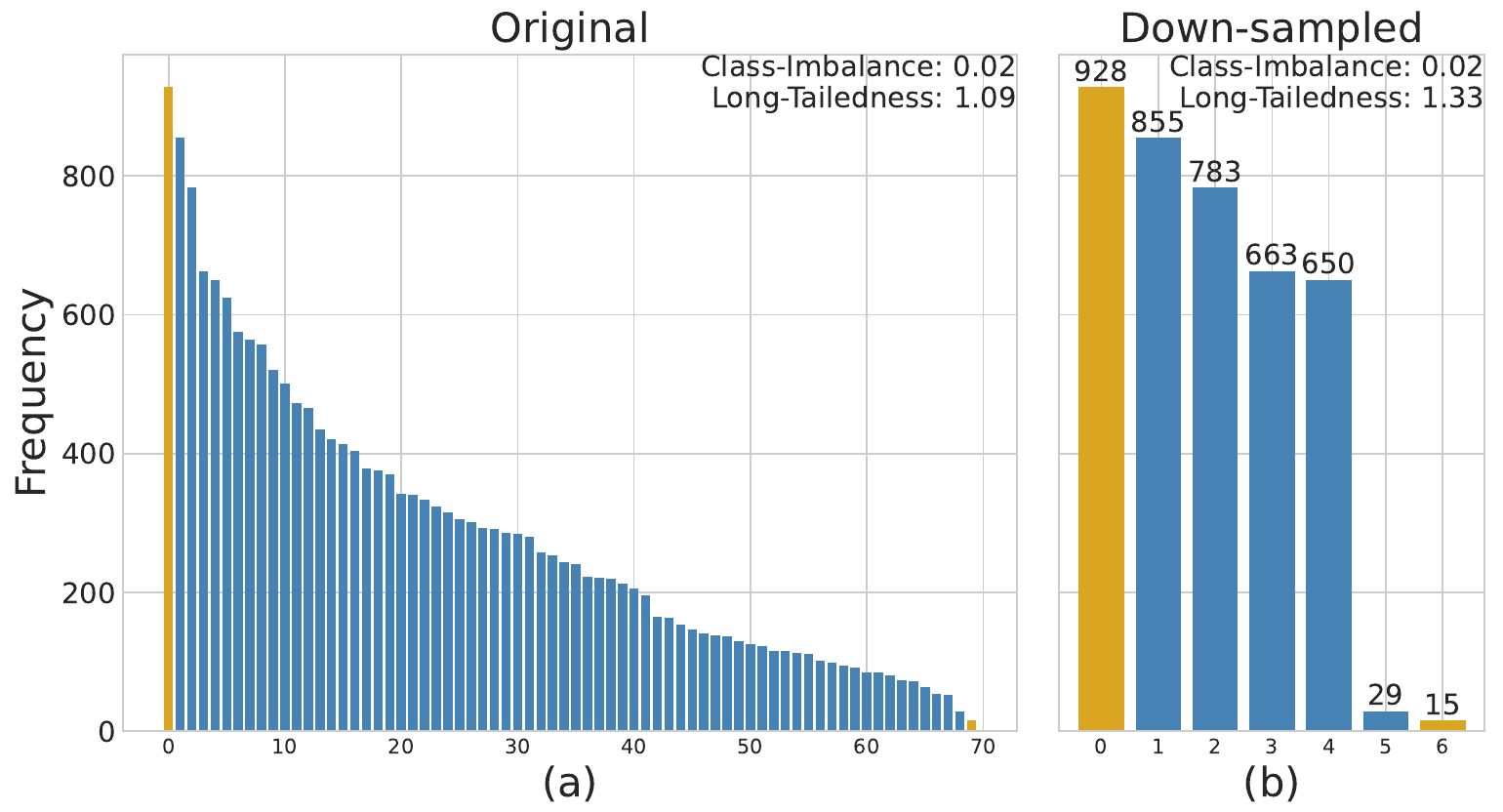}
     \caption{Comparison between two long-tail distribution metrics on (a) the hard case of the original Cora-Full dataset and (b) the easy case of the down-sampled Cora-Full dataset. We observe that the class-imbalance ratio falls short in characterizing the task complexity of two datasets, while the long-tailedness ratio does.}
     \label{fig:ratioExample}
\end{figure}

\noindent\textbf{Long-Tailedness Ratio.} Suppose we are given a graph $\mathcal{G}$ with long-tail class-membership distribution. While Class-Imbalance Ratio~\cite{wu2021graphmixup} measures the imbalance level of observed data, it overlooks the task complexity in the task of long-tail classification. As the number of classes increases, the difficulty of the classification task therefore increases. Taking the Cora-Full collaboration network as shown in Figure~\ref{fig:ratioExample} as an example, we down-sampled 7 classes from the original Cora-Full dataset. Although the class-imbalance ratio remains the same, \ie, 0.02 for both the original and down-sampled datasets, the task complexity varies significantly, \ie, 70 classes in Figure~\ref{fig:ratioExample} (a) v.s 7 classes in Figure~\ref{fig:ratioExample} (b).
In light of this, we introduce a novel quantile-based metric named the long-tailedness ratio to jointly quantify the class-imbalance ratio and task complexity for the long-tail datasets.
The formal definition of the long-tailedness ratio is provided as follows:
\begin{definition}[Long-Tailedness Ratio]\label{def:longtailRatio}
Suppose we have a dataset $\mathcal{D}$ with long-tail classes that follow a descending order in terms of the number of instances. The long-tailedness ratio is 
\begin{equation}
    \texttt{Ratio}_{LT}(p)=\frac{Q(p)}{T-Q(p)},
\end{equation}
where $Q(p)=min\{y: Pr(\mathcal{Y}\leq y)=p, 1\leq y\leq T\}$ is the quantile function of order $p\in (0,1)$ for variable $\mathcal{Y}$, $T$ is the number of classes. The numerator represents the number of classes to which $p$ percent instances belong, and the denominator represents the number of classes to which the else $(1-p)$ percent instances belong in $\mathcal{D}$.
\end{definition}

Essentially, the long-tailedness ratio implies the task complexity of long-tail classification and characterizes two properties of $\mathcal{D}$: (1) class-membership skewness, (2) \# of classes. Intuitively, the higher the skewness of data distribution, the lower the ratio will be; the higher the complexity of the task space (\ie, massive number of classes), the lower the long-tailedness ratio. 
Figure 2 provides a case study on the Cora-Full dataset by comparing the long-tailedness ratio and class-imbalance ratio~\cite{wu2021graphmixup}. In general, we observe that the long-tailedness ratio better characterizes the differences on the original Cora dataset ($\texttt{Ratio}_{LT}(0.8) = 1.09$) and its down-sampled dataset ($\texttt{Ratio}_{LT}(0.8) = 1.33$). In our implementation, we choose $p=0.8$ following the Pareto principle~\cite{pareto1971manual}. 
In Appendix~\ref{sec:ratio}, we additionally offer insights into the utilization of the long-tailedness ratio for the enhanced comprehension of long-tail datasets and as a guiding factor for model selection in practice.

\section{Algorithm}\label{sec:model}

\subsection{Theoretical Analysis}\label{sec:theory}
In this paper, we consider the long-tail problems with data imbalance and massive classes, an area with limited theoretical exploration. For the first time, we propose to reformulate the long-tail problems in the manner of multi-task learning, thereby leveraging the theoretical foundation of multi-task learning to gain insights into long-tail problems. In particular, we view the classification for each class as a learning task\footnote{Here we consider the number of tasks to be the number of classes for simplicity, while in Sec.~\ref{sec:framework} the number of tasks can be smaller than the number of classes after the task grouping operation.} on graph $\mathcal{G}$. A key assumption of multi-task learning is task relatedness, \ie, relevant tasks should share similar model parameters. Similarly, in long-tail learning, we aim to learn the related tasks (classes) concurrently to potentially enhance the performance of each task (classes). We propose to formulate the hypothesis $g$ of long-tail model as $g=\{f_t\}_{t=1}^T\circ h$, where $\circ$ is the  functional composition, $g_t(x) = f_t\circ h(x)\equiv f_t(h(x))$ for each classification task. The function $h: \mathbf{X} \rightarrow \mathbb{R}^{K}$ is the representation extraction function shared across different tasks, $f: \mathbb{R}^{K} \rightarrow \mathbb{R}$ is the task-specific predictor, and $K$ is the dimension of the hidden layer. The training set for the $t^\text{th}$ task $\mathcal{D}_t = \{(\mathbf{x}^{t}_{i}, y^{t}_{i})\}_{i=1}^{n_t}$ contains $n_t$ annotated nodes, $\mathbf{x}^{t}_{i}$ is the $i^\text{th}$ training node in class $t$, and $y^{t}_{i} = t$ for all $i$. The task-averaged risk of representation $h$ and predictors $f_1, \ldots, f_T$ is defined as $\epsilon\left(h, f_1, \ldots, f_T\right)$, and the corresponding empirical risk is defined as $\hat{\epsilon}\left(h, f_1, \ldots, f_T\right)$.
To characterize the performance of head and tail classes in our problem setting, we formally define the loss range of $f_1, \ldots, f_T$ in Definition~\ref{def:range}:
\begin{definition}[Loss Range]
\label{def:range}
The loss range of the $T$ predictors $f_1, \ldots, f_T$ is defined as the difference between the lowest and highest values of the loss function across all tasks.
\begin{equation}
\begin{aligned}
\texttt{Range}(f_1, \ldots, f_T)&=\max_{t}\frac{1}{n_t}\sum_{i=1}^{n_t}l(f_t(h(\mathbf{x}^t_{i})), y^t_{i}) \\
&- \min_{t}\frac{1}{n_t}\sum_{i=1}^{n_t}l(f_t(h(\mathbf{x}^t_{i})), y^t_{i}),
\end{aligned}
\end{equation}
where $l(\cdot, \cdot)$ is a loss function. For the node classification task, $l(\cdot, \cdot)$ refers to the cross-entropy loss.
\end{definition}    

In the scenario of long-tail class-membership distribution, there often exists a tension between maintaining head class performance and improving tail class performance~\cite{zhang2021deep}. Minimizing the losses of the head classes may lead to a biased model, which increases the losses of the tail classes. Under the premise that the model could keep a good performance on head tasks, we conjecture that controlling the loss range could improve the performance on tail tasks and lead to a better generalization performance of the model. 
To verify our idea, we drive the loss range-based generalization error bound for long-tail classes on graphs in the following Theorem~\ref{thm:boundThm}.
\begin{restatable}[Generalization Error Bound]{theorem}{boundThm}
\label{thm:boundThm}
Given the node embedding extraction function $h\in \mathcal{H}$ and the task-specific classifier $f_1, \ldots, f_T \in \mathcal{F}$, with probability at least $1-\delta, \delta \in [0,1]$, we have
\begin{equation}
\begin{aligned}
    \mathcal{E} - \hat{\mathcal{E}}\leq & \sum_{t} \left(\frac{c_{1} \rho R G(\mathcal{H}(\mathbf{X}))}{n_t T} + \sqrt{\frac{9 \ln (2 / \delta)}{2 n_t T^2}}\right.\\
    + & \left. \frac{c_{2} \sup_{h \in \mathcal{H}}\|h(\mathbf{X})\| \texttt{Range}(f_1, \ldots, f_T)}{n_t T} \right),
\end{aligned}
\end{equation}
where $\mathbf{X}$ is the node feature, $T$ is the number of tasks, $n_t$ is the number of nodes in task $t$, $R$ denotes the Lipschitz constant of fuctions in $\mathcal{F}$,  loss function $l(\cdot, \cdot)$ is $\rho$-Lipschitz,
$G(\cdot)$ denotes the Gaussian complexity, and $c_{1}$ and $c_{2}$ are universal constants.
\end{restatable}

\begin{proof}
The proof is provided in Appendix~\ref{sec:proof}.
\end{proof}

\begin{figure*}[t]
  \centering
  \includegraphics[width=\linewidth]{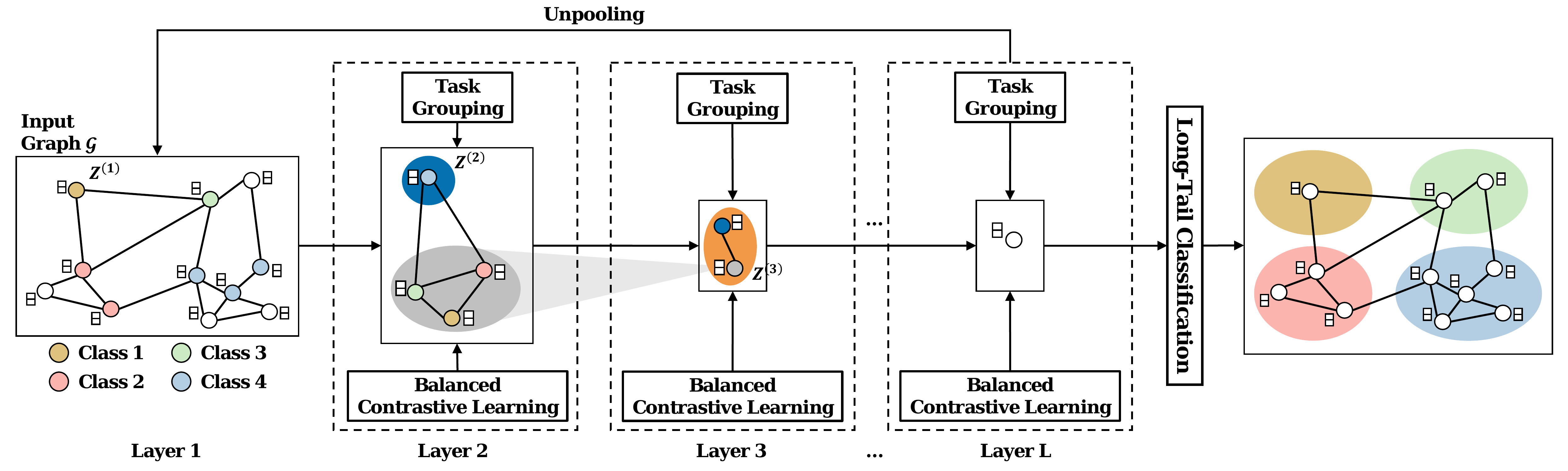}
  \caption{The proposed \name\ framework with $L$ task-grouping layers.}
  \label{fig:framework}
\end{figure*}

\begin{remark}{\#1:}
Theorem~\ref{thm:boundThm} implies that the generalization error is dominated by three key factors, including the Gaussian complexity of the shared representation extraction $h\in \mathcal{H}$, the loss range of the task-specific predictors $f_1, \ldots, f_T$, the number of classes with varying number of samples.
\end{remark}
\begin{remark}{\#2:}
We can derive $\sum_{t}\frac{c_{1} \rho R G(\mathcal{H}(\mathbf{X}))}{n_t T} \geq \frac{T c_{1} \rho R G(\mathcal{H}(\mathbf{X}))}{\sum_{t}n_t} \geq \frac{c_{1} \rho R G(\mathcal{H}(\mathbf{X}))}{\sum_{t}n_t}$ by utilizing Jensen's Inequality. The observation illustrates that when grouping all samples to one task rather than grouping all samples to $T$ tasks, the first term of the upper bound becomes tight. Our conclusion for long-tail learning is different from multi-task learning in that each task corresponds to a fixed number of observed samples~\cite{Maurer16Benefit}. Conversely, in long-tail learning, task complexity is determined by the number of classes $T$, each class exhibiting varying numbers of samples $n_1, \ldots, n_T$. Hence, controlling the complexity of the task space could improve the generalization performance, which motivates the design of the hierarchical task grouping module in Section 3.2.
\end{remark}

\begin{remark}{\#3:}
Reducing the loss range $\texttt{Range}(f_1, \ldots, f_T)$ for all tasks results in a tight second term of the upper bound. This insight inspired the development of long-tail balanced contrastive learning module in Section 3.2, which aims to obtain better task-specific predictors $f_1^{\prime}, \ldots, f_T^{\prime}$ with $\texttt{Range}(f_1^{\prime}, \ldots, f_T^{\prime})<\texttt{Range}(f_1, \ldots, f_T)$. 
\end{remark}

\subsection{\name\ Framework}\label{sec:framework}
The overview of \name\ is presented in Figure~\ref{fig:framework}, which consists of two major modules: M1. hierarchical task grouping and M2. long-tail balanced contrastive learning. 
Specifically, the Remark \#2 of Theorem~\ref{thm:boundThm} inspires that controlling the task complexity with massive and imbalanced classes can potentially improve the generalization performance.
Thus, M1 is designed to control the complexity of task space and capture the information shared across tasks by grouping tasks into the hypertasks to improve overall performance. 
As highlighted in Remark \#3 above, controlling the loss range could improve the generalization performance. Therefore, in M2, we designed a long-tail balanced contrastive loss to balance the head classes and the tail classes.
In the following subsections, we dive into the two modules of \name\ in detail.

\noindent\textbf{M1. Hierarchical Task Grouping.} We propose to address C2 (Label scarcity) and C3 (Task complexity) by leveraging the information learned in one class to help train another class. We implement task grouping to share information across different tasks via hierarchical pooling~\cite{ying2018hierarchical, gao2019graph}, different from previous work which conducts node clustering and ignores the challenges in long-tail learning~\cite{KoCJHLL23grouping}. The core idea of our hierarchical pooling is to choose the important nodes (tasks) and preserve the original connections between the chosen nodes (tasks) and edges to generate a coarsened graph. As shown in Figure~\ref{fig:M1}, the task grouping operation is composed of two steps: \textit{Step 1.} we group nodes (tasks) into several tasks (hypertasks) and \textit{Step 2.} learn the embeddings of the task (hypertask) prototypes. This operation can be easily generalized to the $l^\text{th}$ layers, which leads to the hierarchical task grouping.

Specifically, we first generate a low-dimensional node embedding vector for each node $\mathbf{Z}^{(1)}=(\mathbf{z}^{(1)}_{1}, \ldots, \mathbf{z}^{(1)}_{n})$ via graph convolutional network (GCN)~\cite{kipf2016semi} layers. 
Next, we group nodes into tasks (with the same number of classes) and then group these tasks into hypertasks by stacking several task grouping layers. The $l^\text{th}$ task grouping layer is defined as:
\begin{equation}
\label{equ:gPool}
\begin{aligned}
    &\mathcal{I}  =\texttt{TOP-RANK}(\texttt{PROJ}(\mathbf{Z}^{(l)}), T^{(l)}), \\
    &\mathbf{X}^{(l+1)}  =\mathbf{Z}^{(l)}(\mathcal{I},:) \odot\left(\texttt{PROJ}(\mathbf{Z}^{(l)})\mathbf{1}_{d}^{T}\right), \\
    &\mathbf{A}^{(l+1)}  =\mathbf{A}^{(l)}(\mathcal{I}, \mathcal{I}),
\end{aligned}
\end{equation}
where $l = 1, \ldots L$ is the layer of hierarchical task grouping. We generate a new graph with selected important nodes, where these nodes serve as the prototypes of tasks (hypertasks), and $\mathcal{I}$ is the indexes of the selected nodes. $\texttt{PROJ} (\cdot, \cdot)$ is a projection function to score the node importance by mapping each embedding $\mathbf{z}_i^{(l)}$ to a scalar. $\texttt{TOP-RANK}$ identifies the top $T^{(l)}$ nodes with the highest value after projection. The connectivity between the selected nodes remains as edges of the new graph, and the new adjacency matrix $\mathbf{A}^{(l+1)}$ and feature matrix $\mathbf{X}^{(l+1)}$ are constructed by row and/or column extraction. The subsequent GCN layer outputs the embeddings $\mathbf{Z}^{(l+1)}$ of the new graph based on $\mathbf{X}^{(l+1)}$ and $\mathbf{A}^{(l+1)}$. Notably, $\mathbf{Z}^{(1)}$ is the node embeddings, $\mathbf{Z}^{(2)}$ is the embeddings of the task prototypes corresponding to the classes, and $\mathbf{Z}^{(l)} (l>2)$ is the hypertask prototype embeddings.

The number of tasks $T^{(l)}$ represents the level of abstraction of task grouping, which decreases as the task grouping layer gets deeper. In high-level layers ($l>1$), the number of tasks may be smaller than the number of classes. By controlling $T^{(l)}$, information shared across tasks can be obtained to alleviate the \textit{task complexity}, which is associated with characterizing an increasing number of classes under varying number of samples.
Meanwhile, 
% in a deep level of abstraction, 
nodes that come from different classes with high-level semantic similarities can be assigned to one task. By sharing label information with other different classes within the same hypertask, the problem of \textit{label scarcity} can be alleviated. In layer 2 (Figure~\ref{fig:M1}), we consider a special case of 2 head classes (\ie, class 2 and 4) and 2 tail classes (\ie, class 1 and 3). By grouping the prototypes of classes 1, 2, and 3 into the same hypertask at a later task grouping layer, our method will automatically assign a unique hypertask label to all nodes belonging to the three classes.

In order to well capture the hierarchical structure of tasks and propagate information across different tasks, we need to restore the original resolutions of the graph to perform node classification. Specifically, we stack the same number of unpooling layers as the task grouping layers, which up-samples the features to restore the original resolutions of the graph.
\begin{equation}\label{equ:gUnpool}
\mathbf{X}^{(l+1)}=\texttt{DIST}\left(0_{n \times d}, \mathbf{X}^{(l+1)}, \mathcal{I}\right),
% \vspace{-1pt}
\end{equation}
where $\texttt{DIST}$ restores the selected graph to the resolution of the original graph by distributing row vectors in $X^{(l+1)}$ into matrix $0_{n \times d}$ based on the indices $\mathcal{I}$, $0_{n \times d}$ represents the initially all-zeros feature matrix, $X^{(l+1)} \in \mathbb{R}^{T^{(l)} \times d}$ represents the feature matrix of the current graph, and $\mathcal{I}$ represents the indices of the selected nodes in the corresponding task grouping layer. Finally, the corresponding blocks of the task grouping and unpooling layers are skip-connected by feature addition, and the final node embeddings are passed to an MLP layer for final predictions. 

\begin{figure}[t]
    \centering
    \includegraphics[width=0.9\linewidth]{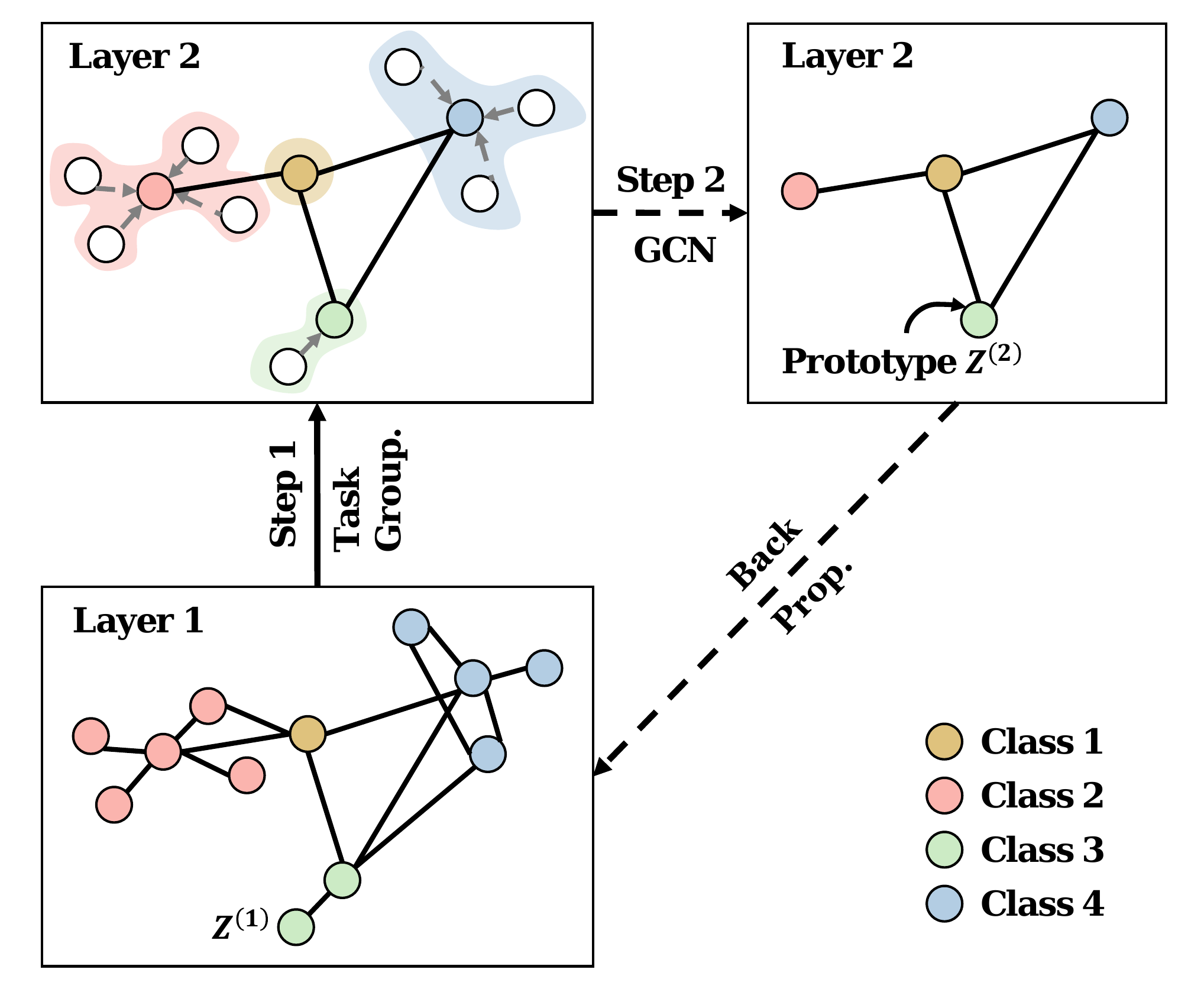}
    \caption{An illustrative figure for M1 with two task-grouping layers. Step 1: nodes are first grouped into four tasks (each representing a class). Step 2: We learn the embeddings of the task prototypes. Finally, the node embeddings are updated by back-propagation.}
    \label{fig:M1}
\end{figure}

\noindent\textbf{M2. Long-Tail Balanced Contrastive Learning.} 
To address C1 (High-skewed data distribution) and C2 (Label scarcity), we propose a principled graph contrastive learning strategy for M1 (Hierarchical task grouping) by passing labels across multiple hierarchical layers. Unlike Graph contrastive learning (GCL)~\cite{xu2021infogcl, hassani2020contrastive, qiu2020gcc, zhu2021an} for learning unsupervised representation of graph data, in this paper, we propose to incorporate supervision signals into each layer of graph contrastive learning. 
Specifically, we employ supervised contrastive loss $\mathcal{L}_{SCL}$ on the labeled node to augment the original graph. It allows joint consideration of head and tail classes, which balances their contributions and alleviates the challenge of \textit{high-skewed data distribution}.
Additionally, we employ balanced contrastive loss $\mathcal{L}_{BCL}$ on each layer of \name. We group all nodes on the graph into several tasks, which facilitates label information to be passed among similar nodes during task grouping. These tasks are subsequently grouped into higher-level hypertasks, which enables label sharing across layers. Through the sharing of label information across nodes and layers, we effectively mitigate the challenge of \textit{label scarcity} in tail classes. 

Next, we introduce supervised contrastive loss $\mathcal{L}_{SCL}$ on the restored original graph. It makes node pairs of the same class close to each other while pairs not belonging to the same class far apart. The mathematical form of the loss function $\mathcal{L}_{SCL}$ on the $i^\text{th}$ node $\mathbf{z}_i$ can be expressed as follows:
\begin{equation}\label{equ:loss_scl}
\begin{aligned}
    \mathcal{L}_{SCL}(\mathbf{z}_i)
    &=-\frac{1}{n_{t}-1} \times \sum_{j\in\mathcal{V}_t\backslash i} \\ &\log \frac{\exp \left(\mathbf{z}_{i} \cdot \mathbf{z}_{j}/\tau\right)}{\sum_{1\leq q\leq T} \frac{1}{n_{q}} \sum_{k\in\mathcal{V}_q} \exp \left(\mathbf{z}_{i} \cdot \mathbf{z}_{k}/\tau\right)}, 
\end{aligned}
\end{equation}
where $\mathbf{z}_i$ belongs to class $t$, $\mathcal{V}_t$ denotes all the nodes belonging to class $t$, $z_k$ represents the embedding of the $k^\text{th}$ node, and temperature $\tau$ controls the strength of penalties on negative node. $\mathcal{L}_{SCL}$ reduces the proportion of contributions from head classes and highlights the importance of tail classes to alleviate the bias caused by high-skewed data distribution.

Moreover, we introduce balanced contrastive loss $\mathcal{L}_{BCL}$ on a coarsened graph, where each node represents a task prototype. For the $l^\text{th}$ task grouping layer, we group tasks in layer $l$ into $T^{(l)}$ hypertasks and calculate the balanced contrastive loss based on the task embeddings $\mathbf{Z}^{(l)}$ and the hypertask prototypes $\mathbf{Z}^{(l+1)}$. It pulls the task embeddings together with their corresponding hypertask prototypes and pushes them away from other prototypes. $\mathcal{L}_{BCL}$ on the $i^\text{th}$ node $\mathbf{z}_i$ can be expressed as follows\footnote{We use the same contrastive loss for each layer. To clarify, we omit layer $(l)$.}:
\begin{equation}\label{equ:loss_bcl}
\begin{aligned}
    \mathcal{L}_{BCL}(\mathbf{z}_i)&=-\frac{1}{n_{t}} \times \sum_{j\in\mathcal{V}_t\backslash i} \\ &\log \frac{\exp \left(\mathbf{z}_{i} \cdot \mathbf{z}_{j}/\tau\right)}{\sum_{1\leq q\leq T} \frac{1}{n_{q}+1} \sum_{k\in\mathcal{V}_q} \exp \left(\mathbf{z}_{i} \cdot \mathbf{z}_{k}/\tau\right)},
\end{aligned}
\end{equation}
where we suppose $\mathbf{z}_i$ belongs to hypertask $t$, here $\mathcal{V}_t$ denotes all the nodes within the $t^\text{th}$ hypertask including the hypertask prototype $\mathbf{z}^{(l+1)}_t$, $n_t$ represents the number of nodes in hypertask $t$, $\mathbf{z}_k=\mathbf{z}^{(l)}_k$ represents the embedding of the $k^\text{th}$ node, and $\tau$ is the temperature. Therefore, $\mathcal{L}_{BCL}$ solves the long-tail classification in two aspects: (1) It potentially controls the range of losses for different tasks. The $n_j + 1$ term in the denominator averages over the nodes of each task so that each task has an approximate contribution for optimizing; (2) The set of $T$ hypertask prototypes is added to obtain a more stable optimization for balanced contrastive learning.
In summary, M2 combines supervised contrastive loss and balanced contrastive loss. With M2, we alleviate the label scarcity by passing label information across all nodes and all layers; and solve the data imbalance by balancing the performance of the head and tail classes.

\noindent\textbf{Overall Objective Function. }Our objective is to minimize the node classification loss (for few-shot annotated data), the unsupervised balanced contrastive loss (for task combinations in each layer), and the supervised contrastive loss (for categories), which is defined as:
\begin{equation}\label{equ:loss}
    \mathcal{L}_{total}=\mathcal{L}_{NC}+\gamma*(\mathcal{L}_{BCL} + \mathcal{L}_{SCL}),
\end{equation}
where $\gamma$ balances the contribution of the three terms. The node classification loss $\mathcal{L}_{NC}$ is defined as follows:
\begin{equation}\label{equ:loss_nc}
    \mathcal{L}_{NC}=\sum_{i=1}^{T}\mathcal{L}_{CE}\left(g(\mathcal{G}), \mathcal{Y}\right),
\end{equation}
where $\mathcal{L}_{CE}$ is the cross-entropy loss, $\mathcal{G}$ represents the input graph with few-shot labeled nodes, and $\mathcal{Y}$ represents the labels. We provide the pseudocode and implementation details in Appendix~\ref{sec:pseudocode}.

\section{Experiments}\label{sec:exp}
To evaluate the effectiveness of \name\ for long-tail classification on graphs, we conduct experiments on six benchmark datasets with a large number of classes and data imbalance. Our model exhibits superior performances compared to various state-of-the-art baselines, as detailed in Section~\ref{sec:performAnalysis}.  
Further, through ablation studies in Section~\ref{sec:ablationStudy}, we demonstrate the necessity of each component of \name. We also report the parameter and complexity sensitivity of \name, which shows that \name\ achieves a convincing performance with minimal tuning efforts and is scalable, as given in Section~\ref{sec:paramComplex}.

\subsection{Experiment Setup}\label{sec:expsetup}
\textbf{Datasets: }We evaluate our proposed framework on Cora-Full~\cite{bojchevski2018deep}, BlogCatalog~\cite{tang20009relational}, Email~\cite{yin2017local}, Wiki~\cite{mernyei2020wiki}, Amazon-Clothing~\cite{mcauley2015inferring}, and Amazon-Electronics~\cite{mcauley2015inferring} datasets to perform node classification task. The first four datasets naturally have smaller $\texttt{Ratio}_{LT}$, which indicates higher long-tail; while the last two datasets have larger $\texttt{Ratio}_{LT}$, which requires the manual process to make them harsh long-tail with $\texttt{Ratio}_{LT}\approx 0.25$. Our proposed $\texttt{Ratio}_{LT}$ reflects a similar trend compared to the class-imbalance ratio but offers a more accurate measurement by considering the total number of classes.
The statistics, the original class-imbalance ratio, and the original long-tailedness ratio ($\texttt{Ratio}_{LT}(0.8)$ as defined in Definition~\ref{def:longtailRatio}) of each dataset are summarized in Table~\ref{tab:datasets}. Further descriptions and details about the additional processing of the six datasets are in Appendix~\ref{sec:detaildata}.
\begin{table}[ht]
\centering
\caption{Dataset statistics.}
\setlength{\tabcolsep}{3pt}
\scalebox{0.82}{
\begin{tabular}{ccccccc}
\hline
  Dataset  & \#Nodes & \#Edges & \#Attributes & \#Classes & Imb. & $\texttt{Ratio}_{LT}$ \\ \hline \hline
  Cora-Full & 19,793 & 146,635 & 8,710 & 70 & 0.016 & 1.09\\
  BlogCatalog & 10,312 & 333,983 & 64 & 38 & 0.002 & 0.77\\
  Email & 1,005 & 25,571 & 128 & 42 & 0.009 & 0.79\\
  Wiki & 2,405 & 25,597 & 4,973 & 17 & 0.022 & 1.00\\
  Amazon-Clothing & 24,919 & 91,680 & 9,034 & 77 & 0.097 & 1.23\\
  Amazon-Electronics & 42,318 & 43,556 & 8,669 & 167 & 0.107 & 1.67 \\
  \hline \\
\end{tabular}
}
\label{tab:datasets}
\end{table}

\noindent\textbf{Comparison Baselines:} We compare \name\ with five imbalanced classification methods and six GNN-based long-tail classification methods.
\begin{desclist}[topsep=-1pt, itemsep=-1pt]
    \item\underline{Classical long-tail learning methods}: Origin utilizes a GCN~\citep{kipf2017semisupervised} as the encoder and an MLP as the classifier. Over-sampling~\citep{chawla2003c4} duplicates the nodes of tail classes and creates a new adjacency matrix with the connectivity of the oversampled nodes. Re-weighting~\citep{yuan2012sampling} penalizes the tail nodes to compensate for the dominance of the head nodes. SMOTE~\citep{chawla2002smote} generates synthetic nodes by feature interpolation tail nodes with their nearest and assigns the edges according to their neighbors' edges. Embed-SMOTE~\citep{ando2017deep} performs SMOTE in the embedding space instead of the feature space.
    \item\underline{GNN-based long-tail learning methods}: 
    GraphSMOTE~\citep{zhao2021graphsmote} extends classical SMOTE to graph data by interpolating node embeddings and connecting the generated nodes via a pre-trained edge generator. It has two variants: $\text{GraphSMOTE}_{T}$ and $\text{GraphSMOTE}_{O}$, depending on whether the predicted edges are discrete or continuous. GraphMixup~\citep{wu2021graphmixup} performs semantic feature mixup and contextual edge mixup to capture graph feature and structure and then develops a reinforcement mixup to determine the oversampling ratio for tail classes. ImGAGN~\citep{qu2021imgagn} is an adversarial-based method that uses a generator to simulate minority nodes and a discriminator to discriminate between real and fake nodes. GraphENS~\citep{park2022graphens} is an augmentation method, synthesizing an ego network for nodes in the minority classes with neighbor sampling and saliency-based node mixing. LTE4G~\citep{yun2022lte4g} splits the nodes into four balanced subsets considering class and degree long-tail distributions. Then, it trains an expert for each balanced subset and employs knowledge distillation to obtain the head student and tail student for further classification.
\end{desclist}

\noindent\textbf{Implementation Details:}
We run all the experiments with 10 random seeds and report the evaluation metrics along with standard deviations. Considering the long-tail class-membership distribution, balanced accuracy (bAcc), Macro-F1, and Geometric Means (G-Means) are used as the evaluation metrics, and accuracy (Acc) is used as the traditional metric.
We provide the details of parameter settings in Appendix~\ref{sec:detailpara}.

\subsection{Performance Analysis}\label{sec:performAnalysis}
\begin{table*}[t]
\caption{Comparison of different methods in node classification task.}
\setlength{\tabcolsep}{3pt}
\centering
\begin{tabular}{cc|cccc|cccc}
\hline
\multirow{2}{*}{} & \multirow{2}{*}{Method} & \multicolumn{4}{c|}{Cora-Full} & \multicolumn{4}{c}{BlogCatalog} \\ \cline{3-10}
 &  & bAcc & Macro-F1 & G-Means & Acc & bAcc & Macro-F1 & G-Means & Acc \\ \hline
\hline
\multirow{5}{*}{\rotatebox{90}{Classical}} & Origin & $52.8\pm0.6$ & $54.5\pm0.7$ & $72.5\pm0.4$ & $62.7\pm0.5$ & $7.1\pm0.4$ & $7.3\pm0.4$ & $26.4\pm0.7$ & $15.1\pm1.0$ \\ 
 & Over-sampling & $52.7\pm0.7$ & $54.4\pm0.6$ & $72.4\pm0.5$ & $62.7\pm0.4$ & $7.1\pm0.3$ & $7.2\pm0.3$ & $26.3\pm0.6$ & $15.1\pm1.2$ \\
 & Re-weight & $52.9\pm0.5$ & $54.4\pm0.5$ & $72.5\pm0.3$ & $62.6\pm0.4$ & $7.2\pm0.4$ & $7.3\pm0.5$ & $26.4\pm0.8$ & $15.1\pm0.8$ \\
 & SMOTE & $52.7\pm0.6$ & $54.4\pm0.5$ & $72.4\pm0.4$ & $62.7\pm0.4$ & $7.1\pm0.4$ & $7.2\pm0.5$ & $26.3\pm0.8$ & $15.3\pm1.2$ \\
 & Embed-SMOTE & $52.9\pm0.5$ & $54.4\pm0.5$ & $73.9\pm0.4$ & $62.6\pm0.4$ & $7.1\pm0.5$ & $7.3\pm0.5$ & $26.3\pm0.9$ & $14.8\pm0.8$ \\\hline
\multirow{5}{*}{\rotatebox{90}{GNN-based}} & $\text{GraphSMOTE}_{T}$ & $54.2\pm0.8$ & $54.7\pm0.8$ & $73.4\pm0.6$ & $62.1\pm0.6$ & $8.6\pm0.4$ & $8.5\pm0.5$ & $28.9\pm0.7$ & $18.3\pm1.1$ \\ 
& $\text{GraphSMOTE}_{O}$ & $54.1\pm0.8$ & $54.5\pm0.7$ & $73.3\pm0.5$ & $62.0\pm0.6$ & $8.6\pm0.4$ & $8.5\pm0.4$ & $28.9\pm0.6$ & $18.3\pm0.9$ \\
 & GraphMixup & $53.9\pm1.3$ & $53.9\pm1.3$ & $73.2\pm0.9$ & $61.4\pm1.2$ & $8.0\pm0.6$ & $7.9\pm0.8$ & $27.9\pm1.2$ & $18.8\pm0.8$ \\
 & ImGAGN & $9.3\pm1.1$ & $6.6\pm1.0$ & $30.2\pm1.9$ & $20.9\pm2.1$ & $6.2\pm0.6$ & $4.9\pm0.5$ & $24.6\pm1.3$ & $20.5\pm1.3$ \\
 & GraphENS & $55.0\pm0.6$ & $54.2\pm0.5$ & $73.9\pm0.4$ & $62.1\pm0.4$ & $9.0\pm0.6$ & $8.9\pm0.5$ & $30.8\pm0.9$ & $12.8\pm1.1$ \\
 & LTE4G & $55.8\pm0.6$ & $54.5\pm0.4$ & $74.5\pm0.4$ & $61.6\pm0.4$ & $6.9\pm0.5$ & $6.7\pm0.6$ & $26.0\pm0.9$ & $11.7\pm1.3$ \\ \hline
 & Ours & $\textbf{55.8}\pm0.5$ & $\textbf{57.1}\pm0.5$ & $\textbf{74.5}\pm0.3$ & $\textbf{64.7}\pm0.7$ & $\textbf{9.8}\pm0.2$ & $\textbf{9.6}\pm0.1$ & $\textbf{30.9}\pm0.4$ & $\textbf{23.2}\pm0.6$ \\ \hline
\hline
\multirow{2}{*}{} & \multirow{2}{*}{Method} & \multicolumn{4}{c|}{Email} & \multicolumn{4}{c}{Wiki} \\ \cline{3-10}
 &  & bAcc & Macro-F1 & G-Means & Acc & bAcc & Macro-F1 & G-Means & Acc \\ \hline
\hline
\multirow{5}{*}{\rotatebox{90}{Classical}} & Origin & $48.9\pm4.5$ & $45.2\pm4.3$ & $69.5\pm3.2$ & $\textbf{66.7}\pm2.1$ & $48.2\pm1.5$ & $49.9\pm1.9$ & $68.6\pm1.1$ & $64.2\pm0.9$ \\ 
 & Over-sampling & $48.4\pm4.2$ & $45.4\pm3.7$ & $69.2\pm3.1$ & $66.4\pm2.0$ & $47.3\pm2.1$ & $48.7\pm2.2$ & $67.9\pm1.5$ & $63.6\pm1.4$ \\
 & Re-weight & $47.9\pm4.6$ & $44.2\pm4.2$ & $68.8\pm3.4$ & $66.3\pm1.7$ & $48.1\pm2.1$ & $49.7\pm2.5$ & $68.5\pm1.6$ & $64.0\pm1.4$ \\
 & SMOTE & $48.4\pm4.2$ & $45.4\pm3.7$ & $69.2\pm3.1$ & $66.4\pm2.0$ & $47.3\pm2.1$ & $48.7\pm2.2$ & $67.9\pm1.5$ & $63.6\pm1.4$ \\
 & Embed-SMOTE & $47.9\pm4.6$ & $44.2\pm4.2$ & $68.8\pm3.3$ & $66.2\pm1.7$ & $48.1\pm2.1$ & $49.7\pm2.5$ & $68.5\pm1.6$ & $63.9\pm1.4$ \\ \hline
 \multirow{5}{*}{\rotatebox{90}{GNN-based}} & $\text{GraphSMOTE}_{T}$ & $43.4\pm2.9$ & $39.1\pm2.8$ & $65.5\pm2.2$ & $60.4\pm1.5$ & $50.3\pm1.7$ & $51.8\pm2.2$ & $70.1\pm1.2$ & $65.8\pm0.9$ \\ 
 & $\text{GraphSMOTE}_{O}$ & $42.3\pm3.1$ & $38.3\pm2.9$ & $64.7\pm2.4$ & $60.1\pm2.3$ & $49.6\pm2.3$ & $51.1\pm2.7$ & $69.6\pm1.7$ & $65.5\pm1.2$ \\
 & GraphMixup & $43.2\pm2.3$ & $38.1\pm2.3$ & $65.4\pm1.7$ & $60.1\pm1.7$ & $50.3\pm2.9$ & $51.2\pm2.9$ & $70.0\pm2.1$ & $65.1\pm1.3$ \\
 & ImGAGN & $27.6\pm3.4$ & $26.8\pm2.9$ & $52.0\pm3.2$ & $46.5\pm3.5$ & $41.2\pm5.7$ & $42.3\pm6.4$ & $63.2\pm4.9$ & $65.5\pm5.8$ \\
 & GraphENS & $50.5\pm3.1$ & $43.7\pm3.3$ & $\textbf{71.1}\pm2.2$ & $62.0\pm2.7$ & $50.8\pm3.3$ & $50.1\pm3.4$ & $70.3\pm2.4$ & $61.7\pm4.4$ \\
 & LTE4G & $46.4\pm2.5$ & $39.3\pm2.4$ & $67.8\pm1.8$ & $57.8\pm3.1$ & $51.0\pm2.9$ & $49.7\pm1.9$ & $70.5\pm2.1$ & $60.4\pm2.1$ \\ \hline
 & Ours & $\textbf{50.5}\pm3.0$ & $\textbf{46.6}\pm3.0$ & $70.7\pm2.1$ & $65.4\pm1.7$ & $\textbf{52.8}\pm2.0$ & $\textbf{54.1}\pm2.3$ & $\textbf{71.9}\pm1.4$ & $\textbf{67.2}\pm1.1$ \\ \hline
\end{tabular}
\label{tab:result}
\end{table*}

\begin{table*}[t]
\caption{Comparison of different methods in node classification task on semi-synthetic long-tail datasets with long-tailedness ratio $\texttt{Ratio}_{LT}(0.8) \approx 0.25$.}
\setlength{\tabcolsep}{3pt}
\centering
\begin{tabular}{cc|cccc|cccc}
\hline
\multirow{2}{*}{} & \multirow{2}{*}{Method} & \multicolumn{4}{c|}{Amazon-Clothing} & \multicolumn{4}{c}{Amazon-Electronics} \\ \cline{3-10}
 &  & bAcc & Macro-F1 & G-Means & Acc & bAcc & Macro-F1 & G-Means & Acc \\ \hline
\hline
\multirow{5}{*}{\rotatebox{90}{Classical}} & Origin & $9.9\pm0.2$ & $9.5\pm0.2$ & $31.3\pm0.3$ & $9.9\pm0.2$ & $16.9\pm0.2$ & $15.2\pm0.2$ & $41.0\pm0.3$ & $16.9\pm0.2$ \\ 
 & Over-sampling & $9.9\pm0.2$ & $9.5\pm0.2$ & $31.3\pm0.3$ & $9.9\pm0.2$ & $16.8\pm0.1$ & $15.1\pm0.1$ & $40.9\pm0.2$ & $16.8\pm0.1$ \\
 & Re-weight & $10.0\pm0.2$ & $9.6\pm0.2$ & $31.4\pm0.3$ & $10.0\pm0.2$ & $17.0\pm0.2$ & $15.2\pm0.2$ & $41.1\pm0.3$ & $17.0\pm0.2$ \\
 & SMOTE & $10.0\pm0.1$ & $9.5\pm0.2$ & $31.4\pm0.2$ & $10.0\pm0.1$ & $16.9\pm0.2$ & $15.1\pm0.2$ & $41.0\pm0.3$ & $16.9\pm0.2$ \\
 & Embed-SMOTE & $9.9\pm0.2$ & $9.5\pm0.2$ & $31.3\pm0.3$ & $9.9\pm0.2$ & $17.0\pm0.2$ & $15.2\pm0.2$ & $41.1\pm0.3$ & $17.0\pm0.2$ \\\hline
\multirow{5}{*}{\rotatebox{90}{GNN-based}} & $\text{GraphSMOTE}_{T}$ & $11.7\pm0.2$ & $10.4\pm0.3$ & $34.0\pm0.3$ & $11.7\pm0.2$ & $18.2\pm0.2$ & $15.6\pm0.2$ & $42.5\pm0.2$ & $18.2\pm0.2$ \\ 
 & $\text{GraphSMOTE}_{O}$ & $11.7\pm0.2$ & $10.4\pm0.3$ & $34.0\pm0.3$ & $11.7\pm0.2$ & $18.2\pm0.2$ & $15.5\pm0.2$ & $42.5\pm0.2$ & $18.2\pm0.2$ \\
 & GraphMixup & $10.9\pm0.5$ & $9.3\pm0.7$ & $32.8\pm0.7$ & $10.9\pm0.5$ & $18.1\pm0.4$ & $15.5\pm0.5$ & $42.5\pm0.5$ & $18.1\pm0.4$ \\
 & ImGAGN & $12.9\pm0.2$ & $9.2\pm0.1$ & $35.7\pm0.2$ & $12.9\pm0.2$ & $13.7\pm0.2$ & $11.0\pm0.0$ & $36.9\pm0.2$ & $13.7\pm0.2$ \\
 & GraphENS & $11.6\pm2.7$ & $10.9\pm2.7$ & $33.6\pm4.3$ & $11.6\pm2.7$ & $19.2\pm3.8$ & $17.2\pm3.6$ & $43.5\pm4.4$ & $19.2\pm3.8$ \\
 & LTE4G & $15.5\pm0.3$ & $16.0\pm0.5$ & $39.1\pm0.3$ & $15.5\pm0.3$ & $20.9\pm0.3$ & $19.9\pm0.3$ & $45.7\pm0.3$ & $20.9\pm0.3$ \\ \hline
 & Ours & $\textbf{17.1}\pm0.5$ & $\textbf{16.8}\pm0.6$ & $\textbf{41.1}\pm0.6$ & $\textbf{17.1}\pm0.5$ & $\textbf{23.6}\pm0.9$ & $\textbf{21.0}\pm1.3$ & $\textbf{48.5}\pm1.0$ & $\textbf{23.6}\pm0.9$ \\ \hline
\end{tabular}
\label{tab:resultSyn}
\end{table*}

\textbf{Overall Evaluation.} We compare \name\ with eleven methods on six real-world graphs, and the performance of node classification is reported in Table~\ref{tab:result} and Table~\ref{tab:resultSyn}.  In general, we have the following observations:
(1) \name\ consistently performs well on all datasets under various long-tail settings and especially outperforms other baselines on harsh long-tail settings (\eg., $\texttt{Ratio}_{LT}(0.8) \approx 0.25$), which demonstrates the effectiveness and generalizability of our model. 
More precisely, taking the Amazon-Electronics dataset (which has 167 classes and follows the Pareto distribution with "80-20 Rule") as an example, the improvement of our model on bAcc (Acc) is 12.9\% compared to the second best model (LTE4G). It implies that \name\ can not only solve the highly skewed data but also capture a massive number of classes.
(2) Classical long-tail learning methods have the worst performance because they ignore graph structure information and only conduct oversampling or reweighting in the feature space. \name\ improves bAcc up to 36.1\% on the natural dataset (BlogCatalog) and 71.0\% on the manually processed dataset (Amazon-Clothing) compared to the classical long-tail learning methods.
(3) GNN-based long-tail learning methods achieve the second-best performance (excluding the Email dataset), which implies that it is beneficial to capture or transfer knowledge on the graph topology, but these models ignore the massive number of classes. In particular, since ImGAGN only considers the high-skewed distribution, as the number of classes increases (from Wiki to Cora-Full), the model becomes less effective. Our model outperforms these GNN-based methods on almost all the natural datasets and metrics (excluding Email), such as up to 12.9\% improvement on the manually processed dataset (Amazon-Electronics).

\noindent\textbf{Performance on Each Class.} To observe the performance of our model for the long-tail classification, we plot the model performance (bAcc) on each class in Figure~\ref{fig:resultPerClass} and for groups of ten classes in Figure~\ref{fig:result10Class} in Appendix~\ref{sec:10classes}. We find that \name\ outperforms the original GCN method (which fails to consider the long-tail class-membership distribution), especially on the tail classes.

\subsection{Ablation Study}\label{sec:ablationStudy} 
Table~\ref{tab:ablation} presents the node classification performance on Cora-Full when considering (a) complete \name\; (b) hierarchical task grouping and node classification loss; and (c) only node classification loss. From the results, we have several interesting observations:
(1) Long-tail balanced contrastive learning module (M2) leads to an increase in bAcc by 1.9\%, which shows its strength in improving long-tail classification by ensuring accurate node embeddings ((a) $>$ (b)). 
(2) Hierarchical task grouping (M1) helps the model better share information across tasks, which achieves impressive improvement on Cora-Full by up to 3.2\% ((b) $>$ (c)).
Overall, the ablation study firmly attests both modules are essential in successful long-tail classification on graphs.
\begin{table}[ht]
    \centering
    \caption{Ablation study on each component of \name.}
    \scalebox{0.92}{
    \begin{tabular}{ccc|cccc}
    \hline
    \multicolumn{3}{c|}{Components} & \multicolumn{4}{c}{Cora-Full} \\
    M1 & M2 & $\mathcal{L}_{CE}$ & bAcc & Macro-F1 & G-Means & Acc \\ \hline
    \hline
    \cmark & \cmark & \cmark & $\textbf{55.8}\pm0.5$ & $\textbf{57.1}\pm0.5$ & $\textbf{74.5}\pm0.3$ & $\textbf{64.7}\pm0.7$ \\
    \cmark & & \cmark & $54.5\pm0.5$ & $56.2\pm0.4$ & $73.6\pm0.3$ & $64.5\pm0.4$ \\
     &  & \cmark & $52.8\pm0.6$ & $54.5\pm0.7$ & $72.5\pm0.4$ & $62.7\pm0.5$ \\
    \hline
    \end{tabular}
    }
    \label{tab:ablation}
\end{table}

\subsection{Parameter and Complexity Analysis}\label{sec:paramComplex}
\textbf{Hyperparameter Analysis.} We configure the number of tasks in the second layer to align with the number of classes in Section~\ref{sec:framework}. To investigate the potential effects of overclustering~\cite{ji2019invariant,kim2022contrastive} where the number of clusters is larger than the number of classes, we conduct experiments by adjusting the number of tasks in the second layer. Table~\ref{tab:overcluster} illustrates the impact of varying the number of tasks on model performance. The experimental results reveal that our model achieves great performance within a certain reasonable range of hyperparameters. However, there is a slight performance degradation when the number of hypertasks is small.
\begin{table}[h]
    \centering
    \caption{Hyperparameter analysis on Cora-Full with respect to the number of tasks in the second layer.}
    \scalebox{1.0}{
    \begin{tabular}{c|cccc}
    \hline
     & \multicolumn{4}{c}{Cora-Full} \\
     & bAcc & Macro-F1 & G-Means & Acc \\ \hline
    \hline
     $[198,70]$ & $55.5$ & $56.7$ & $74.2$ & $64.6$ \\
     $[70,35]$ & $55.8$ & $57.1$ & $74.5$ & $64.7$ \\
     $[2,1]$ & $54.9$ & $56.8$ & $73.9$ & $65.5$ \\
    \hline
    \end{tabular}
    }
    \label{tab:overcluster}
\end{table}

In addition, we study the following hyperparameters: (1) the weight $\gamma$ to balance the contribution of three losses; (2) the temperature $\tau$ of balanced contrastive loss in M2; (3) the activation function in GCN; (4) the number of hidden dimensions; and (5) the dropout rate. The sensitivity analysis results are given in Figure~\ref{fig:hyperparameter} and Figure~\ref{fig:hyperappendix} in Appendix~\ref{sec:parameter}.
Overall, we find \name\ is reliable and not sensitive to the hyperparameters under a wide range. 

\noindent\textbf{Complexity Analysis.} We report the running time and memory usage of \name, GCN, and LTE4G (a efficient state-of-the-art method). 
For better visualize the performance, we run the experiment on an increasing graph size, \ie, from 100 to 100,000 nodes. From Figure~\ref{fig:timecomplexity} in Appendix~\ref{sec:parameter}, we can see the running time of our model is almost the same or superior to LTE4G.
The best space complexity of our method can reach $O(nd+d^2+|\mathcal{E}|)$, which is linear in the number of nodes and the number of edges. The memory usage in several synthetic datasets is given in Figure~\ref{fig:spacecomplexity} in Appendix~\ref{sec:parameter}, illustrating the scalability of our method.

\section{Related Work}\label{sec:relatedWork}
\textbf{Long-tail Problems.} Long-tail data distributions are common in real-world applications~\cite{zhang2021deep}.
Several methods are proposed to solve the long-tail problem, such as data augmentation methods~\cite{chawla2003c4, liu2008exploratory} and cost-sensitive methods~\cite{elkan2001foundations, zhou2005training, yuan2012sampling}. 
However, the vast majority of previous efforts focus on independent and identically distributed (i.i.d.) data, which cannot be directly applied to graph data.
Recently, several related works for long-tail classification on graphs~\cite{qu2021imgagn, zhang23when, park2022graphens, yun2022lte4g, wu2021graphmixup} have attracted attention.
Despite this, the long-tail approaches often lack a theoretical basis. 
The most relevant work lies in imbalanced classification. \citet{cao2019learning} and~\citet{kini2021label} present model-related bounds on the error and the SVM margins, while~\citet{yang2020rethinking} provide the error bound of a linear classifier on data distribution and dimension.
In addition, previous long-tail work is performed under the class imbalance settings where the number of classes can be small, and the number of minority nodes may not be small; but for long-tail learning, the number of classes is large, and the tail nodes are scarce. In this paper, we provide a theoretical analysis of the long-tail problem and conduct experiments on long-tail datasets.

\noindent\textbf{Graph Neural Networks.} Graph neural networks emerge as state-of-the-art methods for graph representation learning, which capture the structure of graphs.
Recently, several attempts have been focused on extending pooling operations to graphs. In order to achieve an overview of the graph structure, hierarchical pooling~\cite{Ma2019graph, ranjan2020asap, lee2019self, ying2018hierarchical, gao2019graph} techniques attempt to gradually group nodes into clusters and coarsen the graph recursively. 
\citet{gao2019graph} propose an encoder-decoder architecture based on gPool and gUnpool layers. 
However, these approaches are generally designed to enhance the representation of the whole graph. In this paper, we aim to explore node classification with the long-tail class-membership distribution via hierarchical pooling methods.

\section{Conclusion}\label{sec:conclusion}
In this paper, we investigate long-tail classification on graphs, which intends to improve the performance on both head and tail classes. By formulating this problem in the fashion of multi-task learning, we propose the generalization bound dominated by the range of losses across all tasks and the task complexity. Building upon the theoretical findings, we also present \name. It is a generic framework with two major modules: M1. Hierarchical task grouping to control the complexity of the task space and address C2 (Label scarcity) and C3 (Task complexity); and M2. Long-tail balanced contrastive learning to control the range of losses across all tasks and solve C1 (High-skewed data distribution) and C2 (Label scarcity). Extensive experiments on six real-world datasets, where \name\ consistently outperforms state-of-art baselines, demonstrate the efficacy of our model for capturing long-tail classes on graphs. 

{\bf Reproducibility:} Our code and data are released at
\url{https://anonymous.4open.science/r/HierTail-B961/}.

\section*{Acknowledgements}
We thank the anonymous reviewers for their constructive comments. This work is supported by 
4-VA, Cisco, Commonwealth Cyber Initiative, DARPA-PA-23-04-01, DHS CINA, Deloitte \& Touche LLP, the National Science Foundation under Award No. IIS-2339989, and Virginia Tech. The views and conclusions are those of the authors and should not be interpreted as representing the official policies of the funding agencies or the government.

% \newpage
\bibliographystyle{ACM-Reference-Format}
\bibliography{reference}

\newpage
\appendix

\section{Symbols and notations}\label{sec:notation}
Here we give the main symbols and notations in this paper.
\begin{table} [htbp]
\caption{Symbols and notations.}
\centering
\begin{tabular}{|l|l|}
\hline Symbol&Description\\
\hline
\hline 
$\mathcal{G}$&input graph.\\
$\mathcal{V}$&the set of nodes in $\mathcal{G}$.\\
$\mathcal{E}$&the set of edges in $\mathcal{G}$.\\
$\mathbf{X}$&the node feature matrix of $\mathcal{G}$.\\
$\mathbf{Z}$&the node embeddings in $\mathcal{G}$.\\
$\mathbf{A}$&the adjacency matrix in $\mathcal{G}$.\\
$\mathcal{Y}$&the set of labels in $\mathcal{G}$.\\
$n$&the number of nodes $|\mathcal{V}|$.\\
$T$&the number of categories of nodes $\mathcal{V}$.\\
\hline
$\texttt{Ratio}_{LT}$&the long-tailedness ratio.\\
\hline
\end{tabular}
\label{TB:Notations}
\end{table}

\section{Details of $\texttt{Ratio}_{LT}(p)$}\label{sec:ratio}
To better characterize class-membership skewness and number of classes, we introduce a novel quantile-based metric named long-tailedness ratio for the long-tail datasets.
\begin{equation*}
    \texttt{Ratio}_{LT}(p)=\frac{Q(p)}{T-Q(p)},
\end{equation*}
where $Q(p)=min\{y: Pr(\mathcal{Y}\leq y)=p, 1\leq y\leq T\}$ is the quantile function of order $p\in (0,1)$ for variable $\mathcal{Y}$, $T$ is the number of categories. The numerator represents the number of categories to which $p$ percent instances belong, and the denominator represents the number of categories to which the else $(1-p)$ percent instances belong in $\mathcal{D}$.

The hyperparameter $p$ allows end users to control the number of classes in the head of the long-tail distribution. If there is no specific definition of the head class in certain domains, we suggest simply following the Pareto principle ($p=0.8$). Using the same $p$ value for two long-tail datasets allows us to compare the complexity. Otherwise, if the $\texttt{Ratio}_{LT}(p)$ of two datasets are measured based on different $p$ values, they are not comparable. If there is a specific definition of the head class in certain domains, we can directly calculate the number of head classes and thus infer the $p$ value.

In addition, in light of class-imbalance ratio and long-tailedness ratio, we gain a better understanding of the datasets and methods to use.
(1) High class-imbalance ratio and low $\texttt{Ratio}_{LT}$ imply high-skewed data distribution, and we may encounter a large number of categories. In such situations, a long-tail method that is designed for data imbalance and an extreme number of classes may be necessary to achieve optimal results.
(2) High class-imbalance ratio and high $\texttt{Ratio}_{LT}$ suggest that the task complexity is low with a relatively small number of categories and the dataset may be imbalanced. Therefore, imbalanced classification approaches such as re-sampling or re-weighting may be effective. 
(3) Low class-imbalance ratio and low $\texttt{Ratio}_{LT}$ imply high task complexity but relatively balanced samples. In such cases, extreme classification methods would be preferred.
(4) Low class-imbalance ratio and high $\texttt{Ratio}_{LT}$ suggest that the dataset may not follow a long-tail distribution, and ordinary machine learning methods may achieve great performance.

\section{Details of Theoretical Analysis}\label{sec:proof}
We obtain the range-based generalization error bound for long-tail categories in the following steps: 
(S1) giving the loss-related generalization error bound based on the Gaussian complexity-based bound in Lemma~\ref{lemma:originBound}; 
(S2) deriving the generalization error bound (Theorem~\ref{thm:boundThm}) related to representation extraction $h$ and the range of task-specific predictors $f_1, \ldots, f_T$ based on the loss-related error bound in S1, the property of Gaussian complexity in Lemma~\ref{lemma:gaussianProp}, and the chain rule of Gaussian complexity in Lemma~\ref{lemma:rangeComplexity}.

First, we have the following assumptions from the previous work~\cite{Maurer16Benefit}.
\begin{assumption}[$R$-Lipschitz Function]
\label{asmp:R-Lip}
Assume each function $f \in \mathcal{F}$ is $R$-Lipschitz in $\ell_{2}$ norm, \ie, $\forall \mathbf{x}, \mathbf{x}^{\prime} \in \mathcal{X}$,
\begin{equation*}
\left|f(\mathbf{x})-f\left(\mathbf{x}^{\prime}\right)\right| \leq R\left\|\mathbf{x}-\mathbf{x}^{\prime}\right\|_{2}.
\end{equation*}
\end{assumption}

\begin{assumption}[$\rho$-Lipschitz Loss]
\label{asmp:rho-Lip}
Assume the loss function $l(\cdot, \cdot)$ is $\rho$-Lipschitz if $\exists~\rho>0$ such that $\forall \mathbf{x} \in \mathcal{X}$, $y, y^{\prime} \in \mathcal{Y}$ and $f, f^{\prime} \in \mathcal{H}$, the following inequalities hold:
\begin{equation*}
\begin{aligned}
    \left|l\left(f^{\prime}(\mathbf{x}), y\right)-l(f(\mathbf{x}), y)\right| &\leq \rho\left|f^{\prime}(\mathbf{x})-f(\mathbf{x})\right|, \\
    \left|l\left(f(\mathbf{x}), y^{\prime}\right)-l(f(\mathbf{x}), y)\right| &\leq \rho\left|y^{\prime}-y\right|.
\end{aligned}
\end{equation*}
\end{assumption}

Based on \citet{Maurer16Benefit}, we can derive the Gaussian complexity-based bound on the training set $\mathbf{X}$ as follows (S1).
\begin{restatable}[Gaussian Complexity-Based Bound]{lemma}{originBound}
\label{lemma:originBound}
Let $\mathcal{F}$ be a class of functions $f: \mathbf{X} \rightarrow[0,1]^{T}$, and $\mathbf{x}^t_{i}$ represents $i^\text{th}$ instances belonging to class $t$. Then, with probability greater than $1-\delta$ and for all $f \in \mathcal{F}$, we have the following bound

\begin{equation}
\begin{aligned}
    &\frac{1}{T} \sum_{t}\left(\mathbb{E}_{\mathbf{X} \sim \mu_{t}}\left[f_{t}(\mathbf{X})\right]-\sum_{i}\frac{1}{n_t} f_{t}\left(\mathbf{x}^t_{i}\right)\right) \\
    \leq &\sum_{t} \left(\frac{\sqrt{2 \pi} G(\mathbf{Z})}{n_t T}+\sqrt{\frac{9 \ln (2 / \delta)}{2 n_t T^2}}\right),
\end{aligned}
\end{equation}
where $\mu_{1}, \ldots, \mu_{T}$ are probability measures, $\mathbf{Z} \subset \mathbb{R}^{n}$ is the random set obtained by $\mathbf{Z}=\left\{\left(f_{t}\left(\mathbf{x}^t_{i}\right)\right): f_t \in \mathcal{F}\right\}$, and $G$ is Gaussian complexity.
\end{restatable}

\begin{proof}
    Following Theorem 8 in~\citep{Maurer16Benefit}, we have $\mathbb{E}_{\mathbf{X} \sim \mu_{t}}[f_{t}(\mathbf{X})]-\sum_{i}\frac{1}{n_t} f_{t}(\mathbf{x}^t_{i}) \leq \frac{\sqrt{2 \pi} G(\mathbf{Z})}{n_t}+\sqrt{\frac{9 \ln (2 / \delta)}{2 n_t}}$. Next, we perform the summation operation for $t$.
\end{proof}
Lemma~\ref{lemma:originBound} yields that the task-averaged estimation error is bounded by the Gaussian complexity in multi-task learning. 
Next, we will give the key property of the Gaussian averages of a Lipschitz image in Lemma~\ref{lemma:gaussianProp}, and will present the chain rule of Gaussian complexity in Lemma~\ref{lemma:rangeComplexity}.
\begin{restatable}[Property of Gaussian Complexity, Corollary 11 in~\citep{Maurer16Benefit}]{lemma}{gaussianProp}
\label{lemma:gaussianProp}
Suppose $\mathbf{Z} \subseteq \mathbb{R}^{n}$ and $\phi: \mathbf{Z} \rightarrow \mathbb{R}^{m}$ is (Euclidean) Lipschitz continuous with Lipschitz constant $R$, we have
\begin{equation}
    G(\phi(\mathbf{Z})) \leq R G(\mathbf{Z}).
\end{equation}
\end{restatable}

\begin{restatable}[Chain Rule of Gaussian Complexity]{lemma}{rangeComplexity}
\label{lemma:rangeComplexity}
    Suppose we have $S=\left\{\left(l(f_{t}(h(X^t_{i})), Y^t_{i})\right): f_t \in \mathcal{F} \right. and \left. h \in \mathcal{H}\right\} \subseteq \mathbb{R}^{n}$. $\mathcal{F}$ is a class of functions $f: \mathbf{Z} \rightarrow \mathbb{R}^{m}$, all of which have Lipschitz constant at most $R$, $\mathbf{Z} \subseteq \mathbb{R}^{n}$ has (Euclidean) diameter $D(\mathbf{Z})$. Then, for any $z_{0} \in \mathbf{Z}$,
\begin{equation*}
\begin{aligned}
    G(S) &\leq c_{1} \rho R G(\mathbf{Z})+c_{2} D(\mathbf{Z}) \texttt{Range}(f_1, \ldots, f_T)\\ & +\rho G\left(\mathcal{F}\left(z_{0}\right)\right),
\end{aligned}
\end{equation*}
where $c_{1}$ and $c_{2}$ are universal constants.
\end{restatable}

\begin{proof}
    % Thm 12 in \citep{Maurer16Benefit}, Thm 1 in \citep{Mao22Task}.
    By the Lipschitz property of the loss function $l(\cdot, \cdot)$ and the contraction lemma~\ref{lemma:gaussianProp}, we have $G(S) \leq \rho G\left(S^{\prime}\right)$, where 
    $S^{\prime}=\left\{\left(f_{t}(h(X^t_{i}))\right): f_t \in \mathcal{F} \text{and } h \in \mathcal{H}\right\} \subseteq \mathbb{R}^{n}$.
    Let 
    \begin{equation}
        R(\mathcal{F})=\sup _{\mathbf{z}, \mathbf{z}^{\prime} \in \mathbf{Z}, \mathbf{z} \neq \mathbf{z}^{\prime}} \mathbb{E}\sup_{f \in \mathcal{F}} \frac{\left\langle\gamma, f(\mathbf{z})-f(\mathbf{z}^{\prime})\right\rangle}{\left\|\mathbf{z}-\mathbf{z}^{\prime}\right\|}.
    \end{equation}
    where $\gamma$ is a vector of independent standard normal variables. Following Theorem 2 in~\citep{Maurer14chain}, we have 
    \begin{equation}\label{equ:ApplyRange}
    \begin{aligned}
        G\left(S\right) \leq & c_{1} \rho R G(\mathcal{H}(\mathbf{X}))+c_{2} \rho D(\mathcal{H}(\mathbf{X})) R(\mathcal{F}) \\ +&\rho \min _{z \in \mathbf{Z}} G(\mathcal{F}(z)).
    \end{aligned}
    \end{equation}
    where $c_{1}$ and $c_{2}$ are constants. Furthermore, 
    
    \begin{equation}
    \begin{aligned}
        &\rho \sup _{\mathbf{z}, \mathbf{z}^{\prime} \in \mathbf{Z}, \mathbf{z} \neq \mathbf{z}^{\prime}} \mathbb{E}\sup_{f \in \mathcal{F}} \frac{\left\langle\gamma, f(\mathbf{z})-f(\mathbf{z}^{\prime})\right\rangle}{\left\|\mathbf{z}-\mathbf{z}^{\prime}\right\|}\\
        =&\sup _{\mathbf{z}, \mathbf{z}^{\prime} \in \mathbf{Z}, \mathbf{z} \neq \mathbf{z}^{\prime}} \frac{\left\|l(f(\mathbf{z}), y)-l(f(\mathbf{z}^{\prime}),y^{\prime})\right\|}{\left\|f(\mathbf{z})-f(\mathbf{z}^{\prime})\right\|} \mathbb{E}\sup_{f \in \mathcal{F}} \frac{\left\langle\gamma, f(\mathbf{z})-f(\mathbf{z}^{\prime})\right\rangle}{\left\|\mathbf{z}-\mathbf{z}^{\prime}\right\|}\\
        \leq &\sup _{\mathbf{z}, \mathbf{z}^{\prime} \in \mathbf{Z}, \mathbf{z} \neq \mathbf{z}^{\prime}} \mathbb{E}\sup_{f \in \mathcal{F}} \frac{\left\langle\gamma, l(f(\mathbf{z}), y)-l(f(\mathbf{z}^{\prime}),y^{\prime})\right\rangle}{\left\|\mathbf{z}-\mathbf{z}^{\prime}\right\|}\\
        \leq & \sup _{\mathbf{z}, \mathbf{z}^{\prime} \in \mathbf{Z}, \mathbf{z} \neq \mathbf{z}^{\prime}} \mathbb{E}\left[\sup_{f \in \mathcal{F}}\left\langle\gamma, l\left(f(\mathbf{z}),y\right)\right\rangle - \sup_{f \in \mathcal{F}}\left\langle\gamma, l\left(f(\mathbf{z}^{\prime}),y^{\prime}\right)\right\rangle \right]\\
        \leq & \sup _{\mathbf{z}, \mathbf{z}^{\prime} \in \mathbf{Z}, \mathbf{z} \neq \mathbf{z}^{\prime}}\left[\frac{1}{n}\sum l(f(h(\mathbf{X})), y) - \frac{1}{n}\sum l(f(h(\mathbf{X}^{\prime})), y^{\prime}) \right]\\
        \leq & \max_{t}\frac{1}{n_t}\sum_{i=1}^{n_t}l(f_t(h(\mathbf{x}^t_{i})), y^t_{i}) -\min_{t}\frac{1}{n_t}\sum_{i=1}^{n_t}l(f_t(h(\mathbf{x}^t_{i})), y^t_{i}).
    \end{aligned}
    \end{equation}
\end{proof}

Finally, we can move to the second step and then derive the generalization error bound related to $h$ and $f_1, \ldots, f_T$ under the setting of long-tail categories on graphs. With the previous assumptions, the generalization bound is given as in the following Theorem~\ref{thm:boundThm}.

\boundThm*
\begin{proof}
    By Lemma~\ref{lemma:originBound}, we have that 
    \begin{equation}\label{equ:ApplyOrigin}
    \mathcal{E} - \hat{\mathcal{E}} \leq \sum_{t} \left(\frac{\sqrt{2 \pi} G(S)}{n_t T}+\sqrt{\frac{9 \ln (2 / \delta)}{2 n_t T^2}}\right),
    \end{equation}
    where $S=\left\{\left(l(f_{t}(h(X^t_{i})), Y^t_{i})\right): f_t \in \mathcal{F} \right. and \left. h \in \mathcal{H}\right\} \subseteq \mathbb{R}^{n}$.  Next, because we have $f_t(0)=0$ for all $f_t \in \mathcal{F}$, the last term in (\ref{equ:ApplyRange}) vanishes. Substitution in (\ref{equ:ApplyRange}) and using Lemma~\ref{lemma:rangeComplexity}, we have
    \begin{equation}
    \begin{aligned}
        G(S) &\leq c_{1} \rho R G(\mathcal{H}(\mathbf{X}))+c_{2} \sqrt{T} D(\mathcal{H}(\mathbf{X})) \texttt{Range}(f_{1}, \ldots, f_{T}).
    \end{aligned}
    \end{equation}
    Finally, we bound $D(\mathcal{H}(\mathbf{X})) \leq 2 \sup _{h}\|h(\mathbf{X})\|$ and substitution in (\ref{equ:ApplyOrigin}), the proof is completed.
\end{proof}

Theorem~\ref{thm:boundThm} shows that the generalization performance of long-tail categories on graphs can be improved by (1) reducing the loss range across all tasks $\texttt{Range}(f_1, \ldots, f_T)$, as well as (2) controlling the task complexity related to $T$. 

\section{Pseudocode}\label{sec:pseudocode}
The pseudo-code of \name\ is provided in Algorithm~\ref{Alg}. Given an input graph $\mathcal{G}$ with few-shot label information $\mathcal{Y}$, our proposed \name\ framework aims to predict $\hat{\mathcal{Y}}$ of unlabeled nodes in graph $\mathcal{G}$. 
We initialize all the task grouping, the unpooling layers and the classifier in Step~\ref{step:init}. Steps~\ref{step:gPoolStart}-\ref{step:gPoolEnd} correspond to the task grouping process: We generate down-sampling graphs and compute node representations using GCNs. Then Steps~\ref{step:gUnpoolStart}-\ref{step:gUnpoolEnd} correspond to the unpooling process: We restore the original graph resolutions and compute node representations using GCNs. An MLP is followed for computing predictions after skip-connections between the task grouping and unpooling layers in Step~\ref{step:mlp}. Finally, in Step~\ref{step:update}, models are trained by minimizing the objective function. In Steps~\ref{step:return}, we return predicted labels $\hat{\mathcal{Y}}$ in the graph $\mathcal{G}$ based on the trained classifier.

\begin{algorithm}[!t]
\caption{The \name\ Learning Framework.}
\label{Alg}
\begin{algorithmic}[1]
\REQUIRE ~~\\
    an input graph $\mathcal{G} = (\mathcal{V}, \mathcal{E}, \mathbf{X})$ with small node class long-tail ratio $\texttt{Ratio}_{LT}(\alpha)$ and few-shot annotated data $\mathcal{Y}$.
\ENSURE ~~\\
    Accurate predictions $\hat{\mathcal{Y}}$ of unlabeled nodes in the graph $\mathcal{G}$\\
    \STATE\label{step:init} Initialize GCNs for graph embedding layer, task grouping layers, and unpooling layers; the MLP for the node classification task in $\mathcal{G}$.
    \WHILE{not converge}\label{step:start}
        \STATE\label{step:embed} Compute node representations in a low-dimensional space of $\mathcal{G}$ via GCN for graph embedding layer.
        \FOR{layer $l \in \{1, \ldots, L\}$}\label{step:gPoolStart}
            \STATE\label{step:gPool} Generate a down-sampling new graph (Eq.~\eqref{equ:gPool}) and compute node representations for the new graph by $l^\text{th}$ task grouping layer.
        \ENDFOR\label{step:gPoolEnd}
        \FOR{layer $l \in \{1, \ldots, L\}$}\label{step:gUnpoolStart}
            \STATE\label{step:gUnpool} Restore the original graph resolutions (Eq.~\eqref{equ:gUnpool}) and compute node representations for the origin graph by $l^\text{th}$ unpooling layer.
        \ENDFOR\label{step:gUnpoolEnd}
        \STATE\label{step:mlp} Perform skip-connections between the task grouping and unpooling layers, and calculate final node embeddings by feature addition. Employ an MLP layer for final predictions.
        \STATE\label{step:update} Calculate node classification loss $\mathcal{L}_{NC}$ (Eq.~\eqref{equ:loss_nc}) with node embeddings obtained in Step~\ref{step:embed}, calculate balanced contrastive loss $\mathcal{L}_{BCL}$ (Eq.~\eqref{equ:loss_bcl}) with node embeddings obtained in Step~\ref{step:gPool}, and calculate supervised contrastive loss $\mathcal{L}_{SCL}$ (Eq.~\eqref{equ:loss_scl}) with node embeddings obtained in Step~\ref{step:mlp}. Update the hidden parameters of GCNs and MLP by minimizing the loss function in Eq.~\eqref{equ:loss}.
    \ENDWHILE\label{step:end}
    \STATE\label{step:return} return predicted labels $\hat{\mathcal{Y}}$ for unlabeled nodes in the graph $\mathcal{G}$.
\end{algorithmic}
\end{algorithm}

\section{Details of Experiment Setup}
\subsection{Datasets}\label{sec:detaildata}
In this subsection, we give further details and descriptions on the six datasets to supplement Sec.~\ref{sec:expsetup}.
(1) Cora-Full is a citation network dataset. Each node represents a paper with a sparse bag-of-words vector as the node attribute. The edge represents the citation relationships between two corresponding papers, and the node category represents the research topic.
(2) BlogCatalog is a social network dataset with each node representing a blogger and each edge representing the friendship between bloggers. The node attributes are generated from Deepwalk following~\citep{Perozzi2014deepwalk}. 
(3) Email is a network constructed from email exchanges in a research institution, where each node represents a member, and each edge represents the email communication between institution members. 
(4) Wiki is a network dataset of Wikipedia pages, with each node representing a page and each edge denoting the hyperlink between pages. 
(5) Amazon-Clothing is a product network which contains products in "Clothing, Shoes and Jewelry" on Amazon, where each node represents a product, and is labeled with low-level product categories for classification. The node attributes are constructed based on the product's description, and the edges are established based on their substitutable relationship ("also viewed"). 
(6) Amazon-Electronics is another product network constructed from products in "Electronics" with nodes, attributes, and labels constructed in the same way. Differently, the edges are created with the complementary relationship ("bought together") between products.

For additional processing, the first four datasets are randomly sampled according to train/valid/test ratios = 1:1:8 for each category. For the last two datasets, nodes are removed until the category distribution follows a long-tail distribution (here we make the head 20\% categories containing 80\% of the total nodes) with keeping the connections between the remaining nodes. We sort the categories by the number of nodes they contain and then downsample them according to Pareto distribution. 
When eliminating nodes, we remove nodes with low degrees and their corresponding edges. After semi-synthetic processing, the long-tailedness ratio of order 0.8 ($\texttt{Ratio}_{LT}(0.8)$) of train set is approximately equal to 0.25. For valid/test sets, we sample 25/55 nodes from each category. Notably, for Amazon-Clothing and Amazon-Electronics, we keep the same number of nodes for each category as test instances, so the values of bAcc and Acc are the same. To sum up, \name\ is evaluated based on four natural datasets, and two additional datasets with semi-synthetic long-tail settings.

\subsection{Parameter Settings}\label{sec:detailpara}
For a fair comparison, we use vanilla GCN as backbone and set the hidden layer dimensions of all GCNs in baselines and \name\ to 128 for Cora-Full, Amazon-Clothing, Amazon-Electronics and 64 for BlogCatalog, Email, Wiki. We use Adam~\citep{Kingma15Adam} optimizer with learning rate 0.01 and weight deacy $5e-4$ for all models. For the oversampling-based baselines, the number of imbalanced classes is set to be the same as in~\citep{yun2022lte4g}. And the scale of upsampling is set to 1.0 as in~\citep{yun2022lte4g}, that is, the same number of nodes are oversampled for each tail category. For GraphSMOTE, we set the weight of edge reconstruction loss to $1e-6$ as in the original paper~\citep{zhao2021graphsmote}. For GraphMixup, we use the same default hyperparameter values as in the original paper~\citep{wu2021graphmixup} except settings of maximum epoch and Adam. For GraphENS~\citep{park2022graphens} and LTE4G~\citep{yun2022lte4g}, we adopt the best hyperparameter settings reported in the paper. For our model, the weight $\gamma$ of contrastive loss is selected in $\{0.01, 0.1\}$, the temperature $\tau$ of contrastive learning is selected in $\{0.01, 0.1, 1.0\}$. We set the depth of the hierarchical graph neural network to 3; node embeddings are calculated for the first layer, the number of tasks is set to the number of categories for the second layer, and the number of tasks is half the number of categories for the third layer. In addition, the maximum training epoch for all the models is set to $10,000$. If there is no additional setting in the original papers, we set the early stop epoch to $1,000$, \ie., the training stops early if the model performance does not improve in 1000 epochs. All the experiments are conducted on an A100 SXM4 80GB GPU.

\section{Additional Experiment Results}
\subsection{Performance on Each Ten Classes}\label{sec:10classes}
 In Figure~\ref{fig:resultPerClass}, we plot the model performance on each category, showing that our model outperforms the original GCN method. In addition, we have added the following figure to provide more details.
\begin{figure}[h]
  \centering
  \scalebox{1.0}{
  \includegraphics[width=\linewidth]{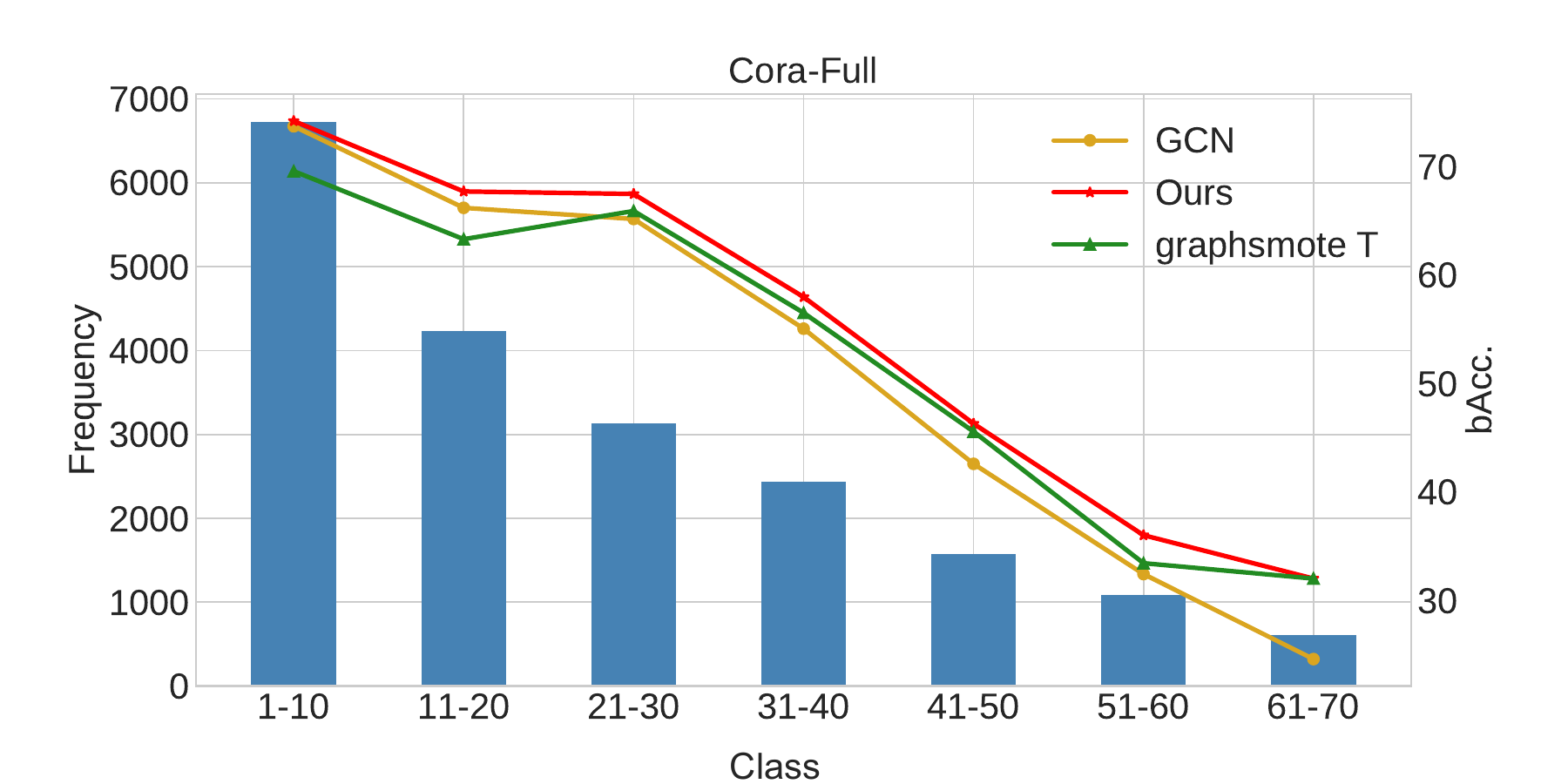}}
  \caption{Performance on groups of ten classes in Cora-Full dataset, where the yellow, red and green curves show bAcc (\%) of GCN, \name\ and GraphSMOTE\_T for node classification.}
  \label{fig:result10Class}
\end{figure}

\subsection{Parameter and Complexity Analysis}\label{sec:parameter}
\textbf{Hyperparameter Analysis:} 
We study the following hyperparameters: (1) the weight $\gamma$ to balance the contribution of three losses; (2) the temperature $\tau$ of balanced contrastive loss in M2; (3) the activation function in GCN; (4) the number of hidden dimensions; and (5) the dropout rate.
First we show the sensitivity analysis with respect to weight $\gamma$ and temperature $\tau$, and the results are shown in Figure~\ref{fig:hyperparameter}. The fluctuation of the bAcc (z-axis) is less than 5\%. The bAcc is slightly lower when both weight $\gamma$ and temperature $\tau$ become larger. 
The analysis results for the remaining hyperparameters are presented in Figure~\ref{fig:hyperappendix}. For analyzing these hyperparameters, all the experiments are conducted with weight $\gamma=0.01$ and temperature $\tau=0.01$.
Overall, we find \name\ is reliable and not sensitive to the hyperparameters under a wide range. 

\begin{figure}[ht]
    \centering
    \includegraphics[width=\linewidth]{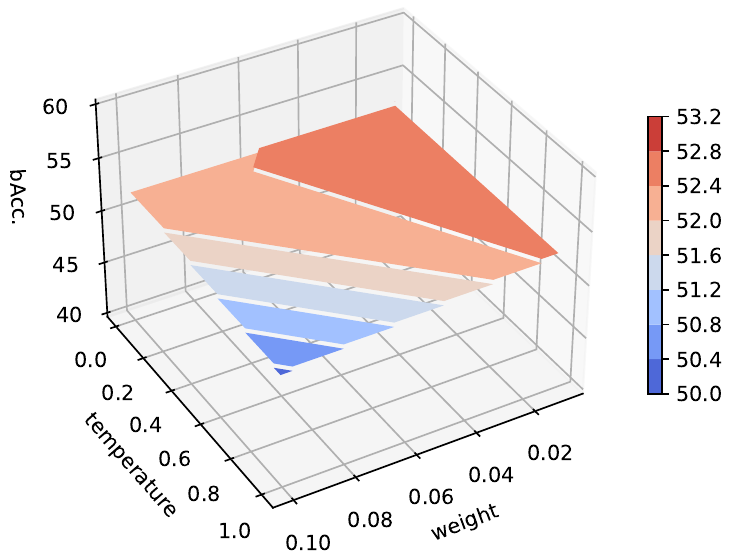}
    \caption{Hyperparameter analysis on Cora-Full with respect to weight $\gamma$ and temperature $\tau$.}
\label{fig:hyperparameter}
\end{figure}

\begin{figure}[ht]
    \centering
    \includegraphics[width=\linewidth]{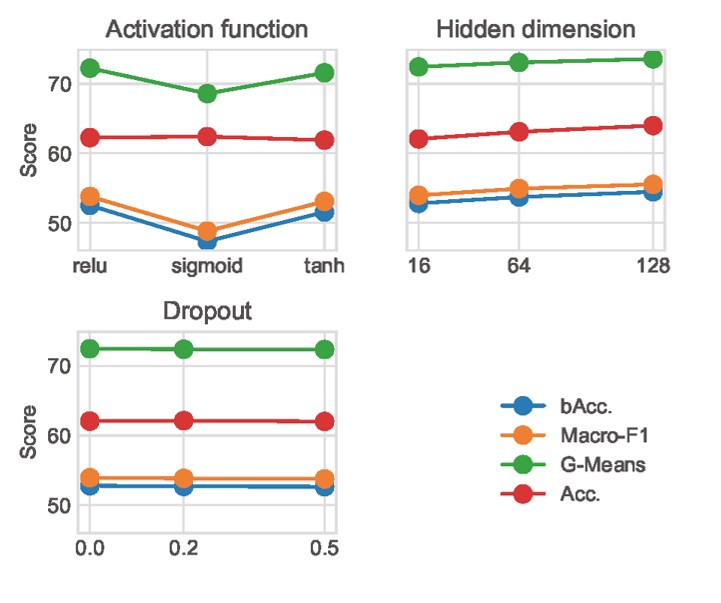}
    \caption{Hyperparameter analysis on Cora-Full.}
    \label{fig:hyperappendix}
\end{figure}

\begin{figure}[!ht]
     \centering
     \includegraphics[width=\linewidth]{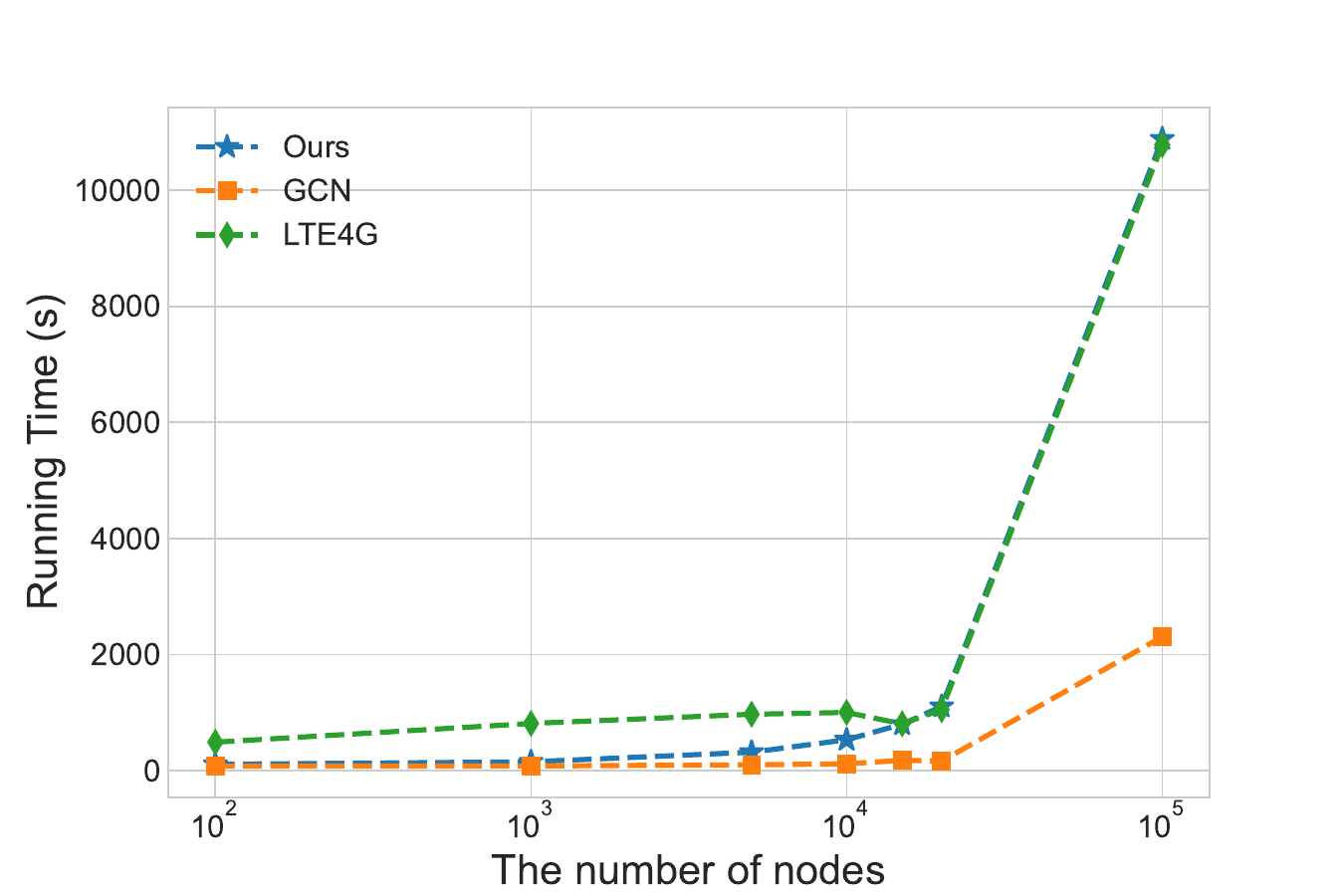}
     \caption{Time complexity analysis w.r.t. the number of nodes.}
     \label{fig:timecomplexity}
\end{figure}

\begin{figure}[!ht]
     \centering
     \includegraphics[width=\linewidth]{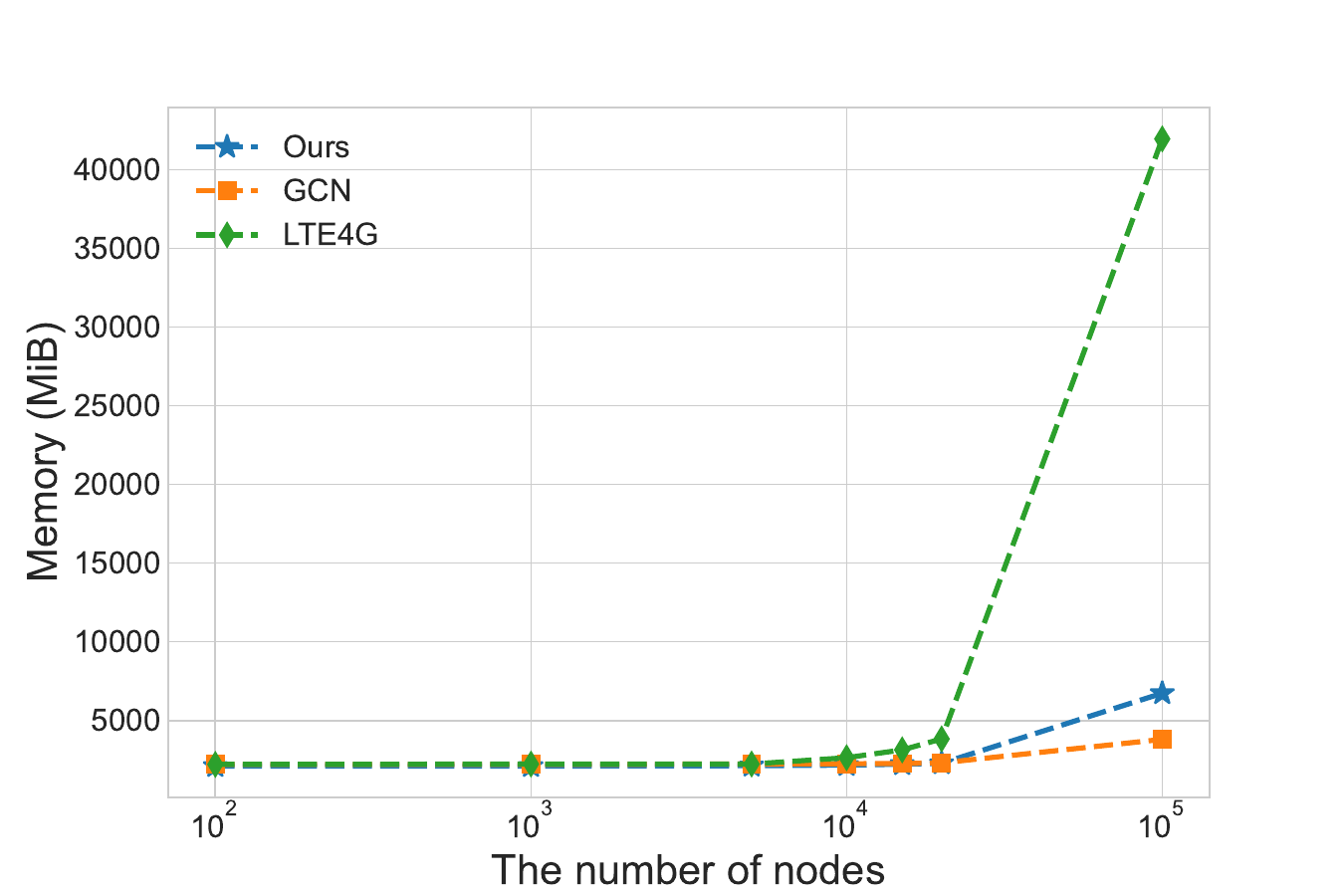}
     \caption{Space omplexity analysis w.r.t. the number of nodes.}
     \label{fig:spacecomplexity}
\end{figure}

\noindent\textbf{Complexity analysis:} We report the running time and memory usage of \name, GCN, and LTE4G (a efficient state-of-the-art method). For better visualization, we conduct experiments on synthetic datasets with an increasing graph size, i.e., from 100 to 100,000 nodes.
As depicted in Figure~\ref{fig:timecomplexity}, our approach \name\ consistently exhibits superior or similar running time compared to the LTE4G method. Although our method has slightly higher running time than GCN, the gap between our approach and GCN remains modest especially when for graph sizes smaller than $10^4$. The relationship between the running time of our model and the number of nodes is similarly linear.
The best space complexity of our method can reach $O(nd+d^2+|\mathcal{E}|)$, which is linear in the number of nodes and the number of edges. From the memory usage given in Figure~\ref{fig:spacecomplexity}, it is shown that \name\ exhibits significantly superior memory usage compared to LTE4G and closely approximates the memory usage of GCN. The results illustrate the scalability of our method.

\end{document}